\documentclass[10pt, conference, letterpaper]{IEEEtran}

\IEEEoverridecommandlockouts
\usepackage{cite}
\usepackage{amsmath,amssymb,amsfonts}
\usepackage{graphicx}
\usepackage{textcomp}
\usepackage{xcolor}
\usepackage{mathabx}
\usepackage{enumitem}

\usepackage[font={footnotesize}]{subcaption}
\usepackage[normalem]{ulem}

\usepackage{cleveref}
\usepackage{mathtools}
\usepackage{amsthm}
\usepackage{graphicx}
\usepackage{tabularx}
\usepackage{url}
\usepackage{multirow}
\usepackage[outdir=./figures]{epstopdf}
\def\cost{\mbox{cost}}
\def\dist{\mbox{dist}}
\def\opt{\mbox{opt}}

\def\argmax{\mathop{\arg\max}}
\def\argmin{\mathop{\arg\min}}

\def\ybf{\mathbf{y}}

\def\xbf{\mathbf{x}}

\newcommand{\ting}[1]{{\leavevmode\color{red}#1}}

\newcommand{\shiqiang}[1]{{\leavevmode\color{orange}#1}}
\newcommand{\hl}[1]{{\leavevmode\color{purple}#1}}

\def\sign{\mbox{sign}}
\newtheorem{definition}{Definition}[section]
\newtheorem{lemma}{Lemma}[section]
\newtheorem{theorem}{Theorem}[section]

\usepackage[lined,ruled,linesnumbered,noend]{algorithm2e}
\usepackage{comment}
\usepackage{verbatim}
\makeatletter
\newcommand{\nosemic}{\renewcommand{\@endalgocfline}{\relax}}
\newcommand{\dosemic}{\renewcommand{\@endalgocfline}{\algocf@endline}}
\let\oldnl\nl
\newcommand{\nonl}{\renewcommand{\nl}{\let\nl\oldnl}}
\makeatother

\def\BibTeX{{\rm B\kern-.05em{\sc i\kern-.025em b}\kern-.08em T\kern-.1667em\lower.7ex\hbox{E}\kern-.125emX}}
\begin{document}

\title{Joint Coreset Construction and Quantization for Distributed Machine Learning
}

 \author{\IEEEauthorblockN{Hanlin Lu\IEEEauthorrefmark{1}, Changchang Liu\IEEEauthorrefmark{2}, Shiqiang Wang\IEEEauthorrefmark{2}, Ting He\IEEEauthorrefmark{1}, Vijay Narayanan\IEEEauthorrefmark{1}, Kevin S. Chan\IEEEauthorrefmark{3}, Stephen Pasteris\IEEEauthorrefmark{4}}
 \IEEEauthorblockA{
 \IEEEauthorrefmark{1}Pennsylvania State University, University Park, PA, USA. Email: \{hzl263,tzh58,vxn9\}@psu.edu\\
 \IEEEauthorrefmark{2}IBM T. J. Watson Research Center, Yorktown, NY, USA. Email: \{Changchang.Liu33,wangshiq\}@us.ibm.com\\
 \IEEEauthorrefmark{3}Army Research Laboratory, Adelphi, MD, USA. Email: kevin.s.chan.civ@mail.mil \\
 \IEEEauthorrefmark{4} University College London, London, UK. Email: stephen.pasteris@gmail.com
 }
\thanks{\scriptsize This research was partly sponsored by the U.S. Army Research Laboratory and the U.K. Ministry of Defence under Agreement Number W911NF-16-3-0001. Narayanan was partly supported by
NSF 1317560. The views and conclusions contained in this document are those of the authors and should not be interpreted as representing the official policies of the U.S. Army Research Laboratory, the U.S. Government, the U.K. Ministry of Defence or the U.K. Government. The U.S. and U.K. Governments are authorized to reproduce and distribute reprints for Government purposes notwithstanding any copyright notation hereon. \looseness=-1
}
 }

\maketitle

\begin{abstract}

Coresets are small, weighted summaries of larger datasets, aiming at providing provable error bounds for machine learning (ML) tasks while significantly reducing the communication and computation costs. To achieve a better trade-off between ML error bounds and costs, we propose the first framework to incorporate quantization techniques into the process of coreset construction. Specifically, we theoretically analyze the ML error bounds caused by a combination of coreset construction and quantization. Based on that, we formulate an optimization problem to minimize the ML error under a fixed budget of communication cost. To improve the scalability for large datasets, we identify two proxies of the original objective function, for which efficient algorithms are developed. For the case of data on multiple nodes, we further design a novel algorithm to allocate the communication budget to the nodes while minimizing the overall ML error. Through extensive experiments on multiple real-world datasets, we demonstrate the effectiveness and efficiency of our proposed algorithms for a variety of ML tasks. In particular, our algorithms have achieved more than 90\% data reduction with less than 10\% degradation in ML performance in most cases.

\end{abstract}


\section{Introduction}

The rapid development of data capturing technologies, e.g., wearables and Internet of Things (IoT), has fueled the explosive growth of data-driven applications that employ various \emph{machine learning (ML)} models to unleash the valuable information hidden in the data. 
One key challenge for such applications is the high communication cost in training ML models over large distributed datasets. One approach to address this challenge is federated learning \cite{konevcny2016federated, smith2017federated, wang2018edge}, where distributed agents iteratively exchange model parameters to collectively train a global model. The exchanged parameters, however, are only useful for a single model, and separate communications are needed to train different models, limiting the efficiency in simultaneously training multiple ML models. When the goal is to train multiple models, which is the focus of our work in this paper, the alternative approach of collecting data summaries at a central location (e.g., a server) is often more efficient, as the summaries can be used to train multiple ML models, amortizing the communication cost.    

To reduce the original dataset into small summaries, several techniques have been proposed, which can be generally classified into: 1) sketching techniques for reducing the feature dimension \cite{phillips2016coresets, feldman2010coresets, sarlos2006improved}; and 2) coreset construction techniques for reducing the sample dimension \cite{Badoiu02STOC, agarwal2005geometric, clarkson2010coresets, har2004coresets, Feldman11STOC, Balcan13NIPS, Feldman16NIPS, Langberg10SODA, Feldman13SODA}. However, sketches change the feature space and thus require adaptations of the ML tasks, e.g., the feature space of a classifier needs to be modified to be applicable to the sketching results.  In comparison, coresets only reduce the cardinality of the datasets and preserve the feature space, making them directly applicable to the original ML tasks. Therefore, we focus on coreset-based data summarization. 

Coresets \cite{Badoiu02STOC, agarwal2005geometric, clarkson2010coresets, har2004coresets, Feldman11STOC} are small, weighted versions of the original dataset, lying in the same feature space. Existing coreset construction algorithms focus on maximally reducing the cardinality with provable guarantees on the ML error. However, most of these algorithms are model-specific, i.e., constructing different coresets for training different ML models, which seriously limits their capability in reducing the communication cost when training multiple ML models. Recently, a \emph{robust coreset construction (RCC)} algorithm was proposed to address this issue \cite{coreset19:report}, where a clustering-based coreset was proved to be applicable for training a variety of ML models with provable error bounds.\looseness=-1 

However, existing coreset construction algorithms only reduce the number of data points, but not the number of bits required to represent each data point. The latter is the goal of quantization \cite{gray1998quantization, lloyd1982least}, where various techniques, from simple rounding-based quantizers to sophisticated vector quantizers, have been proposed to transform the data points from arbitrary values in the sample space to a set of discrete values that can be encoded by a smaller number of bits \cite{Sayood:2012:IDC:3050831, fischer1986pyramid, goodman1975robust, farvardin1987optimal}.   

In this work, we propose the first framework to optimally integrate coreset construction and quantization. 
Intuitively, under a given communication budget specifying the total number of bits to collect, there is a trade-off between collecting more data points at a lower precision and collecting fewer data points at a higher precision.  \emph{Jointly} configuring the quantizer and the coreset construction algorithm to achieve the best trade-off can potentially achieve a smaller ML error than using quantization or coreset construction alone. Our goal is to realize this potential by developing efficient algorithms to compute the optimal configuration parameters explicitly. 


In summary, our contributions include:
\begin{enumerate}
 \item We are the first to incorporate quantization techniques into the coreset construction process. Based on rigorous analysis for the performance of a combination of coreset construction and quantization, we formulate an optimization problem to jointly configure the coreset construction algorithm and the quantizer to minimize the ML error under a given communication budget. 
 \item We propose two algorithms to solve the optimization by identifying proxies of the objective function that can be evaluated efficiently for large datasets. Through theoretical analysis as well as experimental evaluations, we demonstrate the effectiveness of the proposed algorithms in supporting diverse ML tasks.
 \item We further propose a novel algorithm to allocate the communication budget across multiple nodes to adapt our solutions to the distributed setting. 
 Experimental results demonstrate the effectiveness of the proposed algorithm as well as its advantages over existing solutions. \looseness=-1 
\end{enumerate}

\section{Related work}
Coresets have been widely applied in shape fitting and clustering problems \cite{Feldman16NIPS}. Previous coreset construction algorithms can be classified into four categories: sampling-based algorithms \cite{Langberg10SODA, dasgupta2009sampling, lucic2017training, chen2009coresets}, SVD-based algorithms \cite{Feldman13SODA, feldman2013learning, Boutsidis13IT}, space-decomposition algorithms \cite{Har-Peled04STOC, coreset19:report}, and iterative algorithms \cite{Clarkson10TALG, Badoiu02STOC}. Sampling-based algorithms \cite{Langberg10SODA, dasgupta2009sampling, lucic2017training, chen2009coresets} leverage importance/sensitivity sampling and other advanced sampling techniques to select a subset of the original data points to form a coreset. SVD-based algorithms \cite{Feldman13SODA, feldman2013learning, Boutsidis13IT} use the singular value decomposition (SVD) 
to build coresets for ML tasks including dictionary design and sparse representation. Space-decomposition algorithms \cite{Har-Peled04STOC, coreset19:report} partition the original sample space based on different criteria and then select representative points from these partitions. Iterative algorithms \cite{Clarkson10TALG, Badoiu02STOC} select points according to pre-defined criteria to form a coreset in an iterative manner.

However, most existing coreset construction algorithms are model-specific \cite{Feldman11STOC}. That is, 
different coresets will be constructed for training different ML models, increasing the communication cost in collecting the coresets when training multiple ML models. 
To address this issue, \emph{robust coreset construction (RCC)} has been recently proposed in \cite{coreset19:report}, where a single coreset can support a variety of ML models with provable error bounds. 
Therefore, in this work, we focus on RCC as our choice of coreset construction algorithm. 

Quantization techniques \cite{gray1998quantization, lloyd1982least, Sayood:2012:IDC:3050831, fischer1986pyramid, goodman1975robust, farvardin1987optimal,han2015deep, zhou2017incremental, lin2016fixed} aim to quantize the data points to a set of discrete values so that each quantized value can be encoded by a smaller number of bits. Different quantizers have been proposed to support different applications \cite{Sayood:2012:IDC:3050831, fischer1986pyramid, goodman1975robust, farvardin1987optimal}. Recently, quantization has been leveraged to reduce the size of ML models without seriously degrading the model accuracy \cite{han2015deep, zhou2017incremental, lin2016fixed}. 
Existing quantizers can be classified into \emph{scalar quantizers} and \emph{vector quantizers}, where scaler quantizers apply quantization operations to each attribute of a data point, and vector quantizers \cite{gray1990vector, gersho2012vector} apply quantization to each data point as a whole. 
In this work, we focus on a simple rounding-based scalar quantizer due to its simplicity and broad applicability. 
However, we note that our analysis can be easily extended to any given quantizer. 

Despite extensive studies of coreset construction and quantization separately, 
to our knowledge, how to optimally combine them remains an open question. 
To this end, we propose the first framework to integrate coreset construction and quantization, by formulating and solving optimization problems to jointly configure the coreset construction algorithm and the quantizer at hand to achieve the optimal tradeoff between the ML error and the communication cost.  

\textbf{Roadmap.} Section~\ref{sec: Preliminaries} reviews the background on coreset and quantization. Section \ref{sec: Optimizing Tradeoff Between ML Error Bounds and Communication Costs} formulates the main problem. Section \ref{sec: Two Effective Heuristic Solutions} presents two algorithms based on strategic reformulation of the original problem. Section~\ref{sec: distributed setting } extends the solutions to distributed setting. Section~\ref{sec: Performance Evaluation} presents our experimental results. 
At last, Section~\ref{sec:Conclusion} concludes this paper.

\section{Preliminaries}\label{sec: Preliminaries}

In this section, we briefly review several definitions and algorithms that will be used in subsequent sections. 
Frequently used notations in this paper are listed in Table~\ref{tab:variables}.

\begin{table}[tb]
\small
\renewcommand{\arraystretch}{1.2}
\caption{Main notations} \label{tab:variables}
\vspace{-0.7em}
\centering
\begin{tabular}{cl}
  \hline
Variable & Definition  \\
  \hline
  CS & The operation of coreset construction  \\
  
  QT & The operation of quantization \\
  
  $\epsilon$, $\epsilon_i$ & Overall/local ML error \\
  
  $B$, $B_i$ & Global/local communication budget \\
  
  $\mathcal{Y}$, $\mathcal{Y}_i$ & Total/local original dataset \\
  
  $\mathcal{S}$ & Coreset \\
  
  $n$, $k$ & Cardinalities of $\mathcal{Y}$ and $\mathcal{S}$ \\
  
  $\ybf_i$, $y_{ij}$ & One data point and one attribute of the data point\\

  $b_0$, $b$ & $\#$bits for representing each attribute in $\mathcal{Y}$ and $\mathcal{S}$ \\

  $\Delta$ & Maximum quantization error \\
  
  $\xbf$, $\mathcal{X}$ & One solution and solution space for the ML task \\
  
  $\cost(\mathcal{Y}, \xbf)$ & cost function of the ML task \\
  
  $\rho$ & Lipschitz constant for the ML cost function \\
  
  
  $\opt(k)$  & Optimal $k$-means clustering cost for $\mathcal{Y}$\\
  
  $\opt_\infty(k)$ & Optimal $k$-center clustering cost for $\mathcal{Y}$ \\
  
  
  
  
  \multirow{2}{*}{$m_e$} & $\#$ of exponent bits in the floating point representation \\
  & of an attribute  \\
  
  
  $N$ & Number of nodes in distributed setting \\
  \hline
  
\end{tabular}
 \vspace{-.5em}
\end{table}
\normalsize

\subsection{Data Representation}

Let $\mathcal{Y}$ denote the original dataset with cardinality $n := |\mathcal{Y}|$, dimension $d$, and precision $b_0$.  Each data point $\ybf_i\in \mathcal{Y}$ is a column vector in $d$-dimensional space, and each attribute $y_{ij}$ is represented as a floating point number with a sign bit, an $m_e$-bit exponent, and a ($b_0-1-m_e$)-bit significand.  Let $\mathbf{Y} := [\ybf_{1},...,\ybf_{i},...,\ybf_{n}]$ denote the matrix with column vectors $\ybf_{i}$. For simplicity of analysis, we assume that $y_{ij}$'s have been normalized to $[-1,1]$ with zero mean (i.e., $\frac{1}{n}\sum_i y_{ij}=0$ for each $j$). Let $\mu(\mathcal{Y}):= {1\over n}\sum_{\ybf_i\in \mathcal{Y}} \ybf_i$ denote the sample mean of $\mathcal{Y}$. \looseness=-1



\subsection{Coreset Construction}

A generic ML task can be considered as a cost minimization problem. Using $\mathcal{X}$ to denote the set of possible models, and $\cost(\mathcal{Y},\xbf)$ to denote the mismatch between the dataset $\mathcal{Y}$ and a candidate model $\xbf$, the problem seeks to find the model that minimizes $\cost(\mathcal{Y},\xbf)$. The cost function $\cost(\mathcal{Y},\xbf)$ is usually in the form of a summation $\cost(\mathcal{Y}, \xbf) = \sum_{\ybf \in \mathcal{Y}} \cost(\ybf, \xbf)$ or a maximization $\cost(\mathcal{Y},\xbf) = \max_{\ybf \in \mathcal{Y}} \cost(\ybf,\xbf)$, where $\cost(\ybf,\xbf)$ is the per-point cost that is model-specific. For example,  minimum enclosing ball (MEB) \cite{clarkson2010coresets} minimizes a maximum cost and $k$-means minimizes a sum cost. 

A coreset $\mathcal{S}$ is a weighted (and often smaller) dataset that approximates $\mathcal{Y}$ in terms of costs.

\begin{definition}[$\epsilon_\textrm{CS}$-coreset~\cite{Feldman11STOC}]\label{def:epsilon-coreset}
A set $\mathcal{S}\subseteq \mathbb{R}^d$ with weights $u_{\mathbf{q}}$ ($\forall \mathbf{q}\in \mathcal{S}$) is an $\epsilon_\textrm{CS}$-coreset for $\mathcal{Y}$ with respect to (w.r.t.) $\cost(\mathcal{Y},\xbf)$ ($\xbf\in \mathcal{X}$) if $\forall \xbf\in \mathcal{X}$,
\begin{align}\label{eq: epsilon coreset definition}
(1-\epsilon_\textrm{CS})\cost(\mathcal{Y},\xbf) \leq \cost(\mathcal{S},\xbf) \leq (1+\epsilon_\textrm{CS})\cost(\mathcal{Y},\xbf),
\end{align}
where $\cost(\mathcal{S},\xbf)$ is defined as $\cost(\mathcal{S},\xbf) = \sum_{\mathbf{q} \in \mathcal{S}}\hspace{-.25em} u_{\mathbf{q}} \cost(\mathbf{q},\xbf)$ if $\cost(\mathcal{Y},\xbf)$ is a sum cost, and $\cost(\mathcal{S},\xbf) = \max_{\mathbf{q}\in \mathcal{S}}\cost(\mathbf{q},\xbf)$ if $\cost(\mathcal{Y},\xbf)$ is a maximum cost. \looseness=-1
\end{definition}

Definition~\ref{def:epsilon-coreset} also provides a performance measure for coresets: $\epsilon_{\textrm{CS},\mathcal{S}} := \sup_{\xbf\in \mathcal{X}}|\cost(\mathcal{Y},\xbf)-\cost(\mathcal{S},\xbf)|/\cost(\mathcal{Y},\xbf)$ measures the maximum relative error in approximating the ML cost function by coreset $\mathcal{S}$, called the \emph{ML error} of $\mathcal{S}$. The smaller $\epsilon_{\textrm{CS},\mathcal{S}}$, the better $\mathcal{S}$ is in supporting the ML task. \looseness=-1 

Although most coreset construction algorithms only provide guaranteed performance for specific ML tasks, a recent work \cite{coreset19:report} showed that using clustering centers, especially $k$-means clustering centers, as the coreset achieves guaranteed performances for a broad class of ML tasks with Lipschitz-continuous cost functions.  
In the sequel, we denote the optimal $k$-means clustering cost for $\mathcal{Y}$ by $\opt(k)$. It is known that the optimal $1$-means center of $\mathcal{Y}$ is the sample mean $\mu(\mathcal{Y})$.

\subsection{Quantization}\label{subsec: quantizer analysis}

Quantization reduces the number of bits required to encode each data point by transforming it to the nearest point in a set of discrete points, the selection of which largely defines the quantizer. Our solution will utilize the maximum quantization error, defined as $\Delta := \max_{\ybf \in \mathcal{Y}} \dist(\ybf, \ybf')$, where $\ybf'$ denotes the quantized version of data point $\ybf$ and $\dist(\ybf,\ybf')$ is their Euclidean distance. Given a quantizer, $\Delta$ depends on the number of bits used to represent each quantized value. Below we analyze $\Delta$ for a simple but practical \emph{rounding-based quantizer} as a concrete example, but our framework also allows other quantizers. 

Let $y_{ij}$ denote  the $j$-th attribute of the $i$-th data point. 
The $b_0$-bit binary floating point representation of $y_{ij}$ is given by $(-1)^{\sign(y_{ij})} \times 2^{e_{ij}} \times (a_{ij}(0) + a_{ij}(1) \times 2^{-1} + \ldots  + a_{ij}(b_0-1-m_e) \times 2^{-(b_0-1-m_e)})$ \cite{IEEE754}.
Here, $\sign(y_{ij})$ is the sign of $y_{ij}$ ($0$: nonnegative, $1$: negative), $e_{ij}$ is an $m_e$-bit exponent, 
and $a_{ij}(\cdot)\in \{0, 1\}$ are the significant digits, where $a_{ij}(0)\equiv 1$ and does not need to be stored explicitly
. 

Consider a scalar quantizer that rounds each $y_{ij}$ to $s$ significant digits. The quantized value equals
$y'_{ij} = (-1)^{\sign(y_{ij})} \times 2^{e_{ij}} \times (a_{ij}(0) + a'_{ij}(1) \times 2^{-1} + \ldots + a'_{ij}(s) \times 2^{-s})$,
where $a'_{ij}(s)\in \{0,1\}$ is the result of rounding the remaining digits ($0$: round down, $1$: round up). 
As $|y_{ij} - y'_{ij}|\leq 2^{e_{ij}-s}$ and $|y_{ij}| \geq 2^{e_{ij}}$, we have ${|y_{ij} - y'_{ij}|}/{|y_{ij}|} \leq 2^{-s}$. Hence, for $\mathcal{Y}$ in $\mathbb{R}^d$ where each attribute $y_{ij}$ is normalized to $[-1,1]$, the maximum quantization error of this quantizer is bounded by
\begin{align}\label{eq:Delta simplified}
\Delta \leq 2^{-s} \cdot \max_{\ybf_i\in \mathcal{Y}}\|\ybf_i\|. 
\end{align}


\section{Optimal Combination of Coreset Construction and Quantization}\label{sec: Optimizing Tradeoff Between ML Error Bounds and Communication Costs}
In this section, we first analyze the ML error bounds based on the data summary computed by a combination of coreset construction and quantization, and then formulate an optimization problem to minimize the ML error under a given budget of communication cost.

\subsection{Workflow Design}\label{subsec: motivating CS+QT}

The first question in the integration of quantization (QT) into coreset construction (CS) is to determine the order of these two operations. Intuitively, QT is needed after CS since the CS algorithm can result in arbitrary values that cannot be represented using $b$ bits as specified for the quantizer. Therefore, we consider a pipeline where CS is followed by QT. \looseness=-1

\subsection{Error Bound Analysis}\label{subsec: Error bound analysis for CS+QT}

The error bound for CS + QT is stated as follows.

\begin{theorem}\label{thm: coreset + quanization}
After applying a $\Delta$-maximum-error quantizer to an $\epsilon_\textrm{CS}$-coreset $\mathcal{S}$ of the original dataset $\mathcal{Y}$, the quantized coreset $\mathcal{S}'$ is an $(\epsilon_\textrm{CS}+\rho\Delta+\epsilon_\textrm{CS} \cdot \rho\Delta)$-coreset for $\mathcal{Y}$ w.r.t. any cost function satisfying:
\begin{enumerate}
    \item $\cost(\ybf, \xbf) \geq 1$ \label{cond1}
    \item $\cost(\ybf, \xbf)$ is $\rho$-Lipschitz-continuous in $\ybf \in \mathcal{Y}$, $\forall \xbf \in \mathcal{X}$. \label{cond2}
\end{enumerate}
\end{theorem}

Theorem \ref{thm: coreset + quanization} is directly implied by the following Lemma \ref{lem: quantization impact}, which gives the ML error after one single quantization. 

\begin{lemma}\label{lem: quantization impact}
Given a set of points $\mathcal{Y} \subseteq \mathbb{R}^d$, let $\mathcal{Y}'$ be the corresponding set of quantized points with a maximum quantization error of $\Delta$. Then, $\mathcal{Y}'$ is a $\rho\Delta$-coreset of $\mathcal{Y}$ w.r.t. any cost function satisfying the conditions in Theorem~\ref{thm: coreset + quanization}. 
\end{lemma}

\begin{proof}[Proof of Lemma~\ref{lem: quantization impact}]
For each $\ybf \in \mathcal{Y}$, we know $\dist(\ybf, \ybf') \leq \Delta$. By the $\rho$-lipschitz-continuity of $\cost(\cdot,\xbf)$, we have 
\begin{equation}
    | \cost(\ybf, \xbf) - \cost(\ybf', \xbf) | \leq \rho \Delta.
\end{equation}
Moreover, since $\cost(\ybf, \xbf) \geq 1$, we have
\begin{equation}
    \frac{| \cost(\ybf, \xbf) - \cost(\ybf', \xbf) |}{\cost(\ybf, \xbf)} \leq \rho \Delta,
\end{equation}
and thus
\begin{align}
&\hspace{-1em}    (1-\rho\Delta)\cost(\ybf, \xbf) \leq \cost(\ybf', \xbf) \leq (1+\rho\Delta)\cost(\ybf, \xbf). \label{eq:bound, quantize coreset}
\end{align}
If $\cost(\mathcal{Y}, \xbf) = \sum_{\ybf\in \mathcal{Y}} \cost(\ybf, \xbf)$, then treating $\mathcal{Y}'$ as a coreset with unit weights, its cost is $\cost(\mathcal{Y}', \xbf) = \sum_{\ybf'\in \mathcal{Y}'} \cost(\ybf', \xbf)$. Summing (\ref{eq:bound, quantize coreset}) over all $\ybf \in \mathcal{Y}$ (or $\ybf'\in \mathcal{Y}'$), we have
\begin{equation}\label{eq:overall bound, quantize coreset}
    (1 - \rho \Delta) \cost(\mathcal{Y}, \xbf) \leq \cost(\mathcal{Y}', \xbf) \leq (1 + \rho \Delta) \cost(\mathcal{Y}, \xbf). 
\end{equation}
If $\cost(\mathcal{Y}, \xbf) = \max_{\ybf \in \mathcal{Y}} \cost(\ybf, \xbf)$, then the cost of $\mathcal{Y}'$ is $\cost(\mathcal{Y}', \xbf) = \max_{\ybf'\in \mathcal{Y}'} \cost(\ybf', \xbf)$. Suppose that the maximum is achieved at $\ybf_1$ for $\cost(\mathcal{Y}, \xbf)$, and $\ybf'_2$ for $\cost(\mathcal{Y}', \xbf)$. Based on (\ref{eq:bound, quantize coreset}), we have 
\begin{subequations}
\begin{align}
& (1-\rho\Delta)\cost(\ybf_1, \xbf) \leq \cost(\ybf'_1, \xbf) \leq \cost(\ybf'_2, \xbf) \\
& \leq (1+\rho\Delta)\cost(\ybf_2, \xbf)    \leq (1+\rho\Delta)\cost(\ybf_1, \xbf)
\end{align}
\end{subequations}
which again leads to (\ref{eq:overall bound, quantize coreset}) as $\cost(\mathcal{Y}, \xbf) = \cost(\ybf_1, \xbf)$ and $\cost(\mathcal{Y}', \xbf) = \cost(\ybf'_2, \xbf)$.
\end{proof}

\begin{proof}[Proof of Theorem~\ref{thm: coreset + quanization}]
By Lemma~\ref{lem: quantization impact} and Definition~\ref{def:epsilon-coreset}, we have that 
$\forall \xbf \in \mathcal{X}$, 
\begin{align}
    & (1\hspace{-.15em}-\hspace{-.15em}\epsilon_\textrm{CS})(1\hspace{-.15em}-\hspace{-.15em}\rho\Delta)\cost(\mathcal{Y},\xbf) \hspace{-.15em} \leq \hspace{-.15em} (1\hspace{-.15em}-\hspace{-.15em}\rho\Delta)\cost(\mathcal{S}, \xbf) \hspace{-.15em} \nonumber \\  
    & \leq \hspace{-.15em}  \cost(\mathcal{S}', \xbf) \leq (1+\rho\Delta) \cost(\mathcal{S}, \xbf) \nonumber\\
    & \leq (1+\epsilon_\textrm{CS})(1+\rho\Delta) \cost(\mathcal{Y}, \xbf), \label{eq: cs+qt error bound}
\end{align}
which yields the result.
\end{proof}

\subsection{Configuration Optimization}

\subsubsection{Abstract Formulation}

Our objective is to minimize the ML error under bounded communication costs, through the joint configuration of coreset construction and quantization. Given a $n$-point dataset in $\mathbb{R}^d$ and a communication budget of $B$, we aim to find a quantized coreset $\mathcal{S}$ with $k$ points and a precision of $b$ bits per attribute, that can be represented by no more than $B$ bits.  
Our goal is to \emph{Minimize the Error under a given Communication Budget (MECB)}, formulated as
\begin{subequations}\label{prob: general}
\begin{align}
\min_{b,k} \quad & \epsilon_\textrm{CS}(k)+\rho\Delta(b)+\epsilon_\textrm{CS}(k)\cdot\rho\Delta(b) \label{MECB:obj}\\ 
\textrm{s.t.} \quad &  b \cdot k \cdot d \leq  B,\\
           & b,\:  k \in \mathbb{Z}^+,
\end{align} 
\end{subequations}
where $\epsilon_\textrm{CS}(k)$ represents the ML error of a $k$-point coreset constructed by the given coreset construction algorithm, and $\Delta(b)$ is the maximum quantization error of $b$-bit quantization by the given quantizer. We want to find the optimal values of $k$ and $b$ to minimize the error bound (\ref{MECB:obj}) according to Theorem~\ref{thm: coreset + quanization}, under the given budget $B$. Note that our focus is on finding the optimal configuration of known CS/QT algorithms instead of developing new algorithms. 

\subsubsection{Concrete Formulation}

We now concretely formulate and solve an instance of MECB for two practical CS/QT algorithms. 
Suppose that the CS operation is by the $k$-means based \emph{robust coreset construction (RCC)} algorithm in \cite{coreset19:report}, which is proven to yield a $\rho\sqrt{f(k)}$-coreset for all ML tasks with $\rho$-Lipschitz-continuous cost functions, where $f(k) := \opt(k) - \opt(2k)$ is the difference between the $k$-means and the $2k$-means costs. Moreover, suppose that the QT operation is by the rounding-based quantizer defined in Section~\ref{subsec: quantizer analysis}, which has a maximum quantization error of $\Delta(b) := 2^{-(b-1-m_e)}\max_{\ybf_i\in \mathcal{Y}}\|\ybf_i\|$ to generate a $b$-bit quantization with $s=b-1-m_e$ significant digits according to (\ref{eq:Delta simplified}).
Then, by Theorem~\ref{thm: coreset + quanization}, the MECB problem in this case becomes:
\begin{subequations}\label{prob: concrete}
\begin{align}
\min_{b,k} \quad & \rho \sqrt{f(k)} + \rho \Delta(b) + \rho^2 \Delta(b) \sqrt{f(k)} \label{MD: obj}\\ 
\textrm{s.t.} \quad & b \cdot k \cdot d \leq  B, \label{HM-RCC:B}\\
           & b, k \in \mathbb{Z}^+. 
\end{align} 
\end{subequations}

\subsubsection{Straightforward Solution} \label{subsubsec:baselineEM}

In (\ref{prob: concrete}), only $b$ (or $k$) is the free decision variable. Thus, a straightforward way to solve (\ref{prob: concrete}) is to evaluate the objective function\footnote{ We note that it is NP-hard to compute the $k$-means costs $\opt(k)$ and $\opt(2k)$ in evaluating $f(k)$. Nevertheless, one can compute an approximation using existing $k$-means heuristics, e.g., \cite{Arthur07SODA}. } (\ref{MD: obj}) for each possible value of $b$ and then select $b^*$ that minimizes the objective value. We refer to this solution as the \emph{EMpirical approach (EM)} later in the paper. However, this approach is computationally expensive for large datasets, as evaluating $f(k)$ requires solving $k$-means problems for large values of $k$. 
To address this challenge, we will develop efficient heuristic algorithms for approximately solving (\ref{prob: concrete}) in the following section by identifying proxies of the objective function that can be evaluated efficiently.  

\section{Efficient Algorithms for MECB 
}\label{sec: Two Effective Heuristic Solutions}
In this section, we propose two algorithms to effectively and efficiently solve 
the concrete MECB problem given in (\ref{prob: concrete}). 

\subsection{Eigenvalue Decomposition Based Algorithm for MECB (EVD-MECB)}\label{subsec: Eigenvalue Decomposition Based Heuristic algorithm}

\subsubsection{Re-formulating the Optimization Problem}

This algorithm is motivated by the following bound derived in \cite{ding2004k}. 

\begin{theorem}[Bound for $k$-means costs \!\!\cite{ding2004k}]
The optimal $k$-means cost for $\mathcal{Y}$ is bounded by
\begin{align}
\opt(k) \geq n \widebar{\ybf^2} - \sum_{i = 1}^{k-1} \lambda_i,    
\label{eq:EVDLowerBound}
\end{align}
where $n \widebar{\ybf^2} := \sum_{i=1}^n \ybf_i^T\ybf_i$ is the total variance and $\lambda_i$ is the $i$-th principal eigenvalue of the covariance matrix $\mathbf{Y}\mathbf{Y}^T$.
\end{theorem}

\begin{algorithm}[tb]
\caption{EVD-MECB}
\label{Alg:heuristic kmeans coreset}
\small
\SetKwInOut{Input}{input}\SetKwInOut{Output}{output}
\Input{A dataset $\mathcal{Y}$, Lipschitz constant $\rho$ for the targeted ML task, communication budget $B$. }
\Output{Optimal $(k^*, b^*)$ to configure a quantized $\epsilon$-coreset $\mathcal{S}'$ for $\mathcal{Y}$ within budget $B$.}
\BlankLine
Calculate eigenvalues $\{\lambda_i\}_{i=1}^d$ for $\mathbf{Y} \mathbf{Y}^T$\;
$\Lambda_j \leftarrow \sum_{i=1}^j \lambda_i, \forall 1 \leq j \leq d$\;
\ForEach{$b=[1+m_e,\: 2 + m_e, \ldots, b_0]$}
{
    $k \leftarrow \Bigl\lfloor B/d/b \Bigr\rfloor$\;
    $f(k) \leftarrow \Lambda_{2k-1} - \Lambda_{k-1}$\;
    $\Delta(b) \leftarrow 2^{-(b-1-m_e)}\max_{\ybf_i\in \mathcal{Y}}\|\ybf_i\|$\;
    $\epsilon(k, b) \leftarrow \rho \cdot \sqrt{f(k)} + \rho \cdot \Delta(b) + \rho^2 \cdot \Delta(b) \cdot \sqrt{f(k)} $\;
}
$(k^*, b^*) \leftarrow \argmin{\epsilon(k,b) }$\;
\textbf{return} $(k^*, b^*)$\;
\end{algorithm}
\normalsize
\setlength{\textfloatsep}{.5em}

We use the bound in (\ref{eq:EVDLowerBound}) as an approximation of $k$-means cost that is much faster to evaluate than the exact $k$-means cost.
Replacing $\opt(k)$ by this bound, we obtain an approximation of~(\ref{prob: concrete}), where $f(k)$ is approximated by
\begin{align}
f(k) \approx n \widebar{\ybf^2} - \sum_{i = 1}^{k-1} \lambda_i - (n \widebar{\ybf^2} - \sum_{i = 1}^{2k-1} \lambda_i) = \sum_{i = k}^{2k-1} \lambda_i. \label{eq:EVD-new constraint}
\end{align}

\subsubsection{EVD-MECB Algorithm}\label{subsec: Heuristic for HM-RCC}

The righthand side of (\ref{eq:EVD-new constraint}) is easier to calculate than the exact value of $f(k)$, as we can compute the eigenvalue decomposition once \cite{stewart2002krylov}, and use the results to evaluate $\sum_{i = k}^{2k-1} \lambda_i$ for all possible values of $k$. As each number in $\mathcal{Y}$ has $b_0 - 1 - m_e$ significant digits, the number of feasible values for $b$ (and hence $k$) is $b_0 - 1 - m_e$. By enumerating all the feasible values, we can easily find the optimal solution $(k^*, b^*)$ to this approximation of (\ref{prob: concrete}). We summarize the algorithm in Algorithm~\ref{Alg:heuristic kmeans coreset}.

\begin{algorithm}[tb]
\caption{$k$-center cost computation }
\label{Alg:greedy k-center}
\small
\SetKwInOut{Input}{input}\SetKwInOut{Output}{output}
\Input{A dataset $\mathcal{Y}$, the maximum number of centers $K$. }
\Output{The costs $(g(k))_{k=1}^K$ for greedy $k$-center clustering for $k=1,\ldots,K$.}
\BlankLine
$\mathcal{G} \leftarrow \emptyset$\;
\ForEach{$\ybf \in \mathcal{Y}$}
{
    $d(y) \leftarrow \infty$\;
}
\While{$|\mathcal{G}| < K$}
{
    find $\ybf \leftarrow \argmax_{\mathbf{q} \in \mathcal{Y}}{d(\mathbf{q})}$\; \label{kcenter:1} 
    $\mathcal{G} \leftarrow \mathcal{G} \cup \{\ybf\}$\; \label{kcenter:2} 
    \ForEach{$j = 1,\ldots,|\mathcal{Y}|$}
    {
        $d(\ybf_j) \leftarrow \min(d(\ybf_j),\: \dist(\ybf_j, \ybf))$\; \label{kcenter:3}
    }
    $g(|\mathcal{G}|) \leftarrow \max_{\ybf_j\in \mathcal{Y}} d(\ybf_j)$\; \label{kcenter:4} 
}
\textbf{return} $(g(k))_{k=1}^K$\; 
\end{algorithm}
\normalsize
\setlength{\textfloatsep}{.5em}

\subsection{Max-distance Based Algorithm for MECB (MD-MECB)}\label{subsec: Max-distance based Heuristic algorithm}
\subsubsection{Re-formulating the Optimization Problem}\label{Error bound analysis for max-distance-based CS+QT}

Alternatively, we can bound the ML error based on the maximum distance between each data point and its corresponding point in the coreset. Let $f_2(k) := \max_{i=1,\ldots,k}\max_{\ybf\in \mathcal{Y}_i} \dist(\ybf, \mu(\mathcal{Y}_i))$ be the maximum distance between any data point and its nearest $k$-means center, where $\{\mu(\mathcal{Y}_i)\}_{i=1}^k$ are the $k$-means clusters. Then, a similar proof as that of Lemma~\ref{lem: quantization impact} implies the following.

\begin{lemma}\label{lem:distance-based coreset}
The centers of the optimal $k$-means clustering of $\mathcal{Y}$, each weighted by the number of points in its cluster,  provide a $\rho f_2(k)$-coreset for $\mathcal{Y}$ w.r.t. any cost function satisfying the conditions in Theorem~\ref{thm: coreset + quanization}.
\end{lemma}

This lemma provides an alternative error bound for the RCC algorithm in \cite{coreset19:report}, which constructs the coreset as in Lemma~\ref{lem:distance-based coreset}. Using $\epsilon_\textrm{CS} = \rho f_2(k)$, if we apply the rounding-based quantization after RCC, we can apply Theorem~\ref{thm: coreset + quanization} to obtain an alternative error bound for the resulting quantized coreset, which is $\rho  f_{2}(k) + \rho  \Delta(b) + \rho^2  f_{2}(k)  \Delta(b)$.
We note that minimizing $f_2(k)$ is exactly the objective of \emph{$k$-center clustering} \cite{lim2005k, khuller2000capacitated, khuller2000fault}. Hence, we would like to use the optimal $k$-center cost, denoted by $\opt_\infty(k)$, as a heuristic to calculate $f_2(k)$. By using this alternative error bound and approximating $f_2(k) \approx \opt_\infty(k)$, we can reformulate the MECB problem as follows:
\begin{subequations}\label{eq:MD-MECB}
\begin{align}
\min_{b,k} \quad &  \rho \cdot  \opt_\infty(k) + \rho  \Delta(b) + \rho^2 \cdot  \opt_\infty(k)  \Delta(b) \label{EVD: obj}\\ 
\textrm{s.t.} \quad &  b \cdot k \cdot d \leq  B, \label{framework:B}\\
           & b, k \in \mathbb{Z}^+,
\end{align} 
\end{subequations}
where $\Delta(b)$ is defined as in (\ref{prob: concrete}).

\subsubsection{MD-MECB Algorithm}\label{subsec: Greedy algorithm for GM-RCC}
The re-formulation (\ref{eq:MD-MECB}) allows us to leverage algorithms for $k$-center clustering to efficiently evaluate $\opt_\infty(k)$. 
Although $k$-center clustering is a NP-hard problem \cite{hochbaum1985best}, a number of efficient heuristics have been proposed. In particular, it has been proved \cite{hochbaum1985best} that the best possible approximation for $k$-center clustering is $2$-approximation, achieved by a simple greedy algorithm \cite{kleinberg2006algorithm} that keeps adding the point farthest from the existing centers to the set of centers until $k$ centers are selected. 
The beauty of this algorithm is that we can modify it to compute the $k$-center clustering costs for all possible values of $k$ in one pass, as shown in Algorithm~\ref{Alg:greedy k-center}. Specifically, after adding each center (lines~\ref{kcenter:1}--\ref{kcenter:2}) and updating the distance from each point to the nearest center (line~\ref{kcenter:3}), we record the clustering cost (line~\ref{kcenter:4}). As the greedy algorithm achieves $2$-approximation \cite{hochbaum1985best}, the returned costs satisfy $\opt_\infty(k) \leq g(k) \leq 2 \opt_\infty(k)$, where $g(\cdot )$ is defined in line~\ref{kcenter:4}. 
Based on this algorithm, the MD-MECB algorithm, shown in Algorithm~\ref{Alg:greedy max-distance coreset}, solves an approximation of (\ref{eq:MD-MECB}) with $\opt_\infty(k)$ approximated by $g(k)$. 

\begin{algorithm}[tb]
\caption{MD-MECB}
\label{Alg:greedy max-distance coreset}
\small
\SetKwInOut{Input}{input}\SetKwInOut{Output}{output}
\Input{A dataset $\mathcal{Y}$, Lipschitz constant $\rho$ for the targeted ML task; communication budget $B$. }
\Output{Optimal $(k^*, b^*)$ to configure a quantized $\epsilon$-coreset $S'$ for $Y$ within budget $B$.}
Run Algorithm \ref{Alg:greedy k-center} with input  $\mathcal{Y}$ and $K\!=\!\min\{\bigl\lfloor \frac{B}{d \cdot (1+m_e)} \bigr\rfloor, n\}$\;
\ForEach{$b=[1+m_e,\: 2+m_e, \ldots, b_0]$}
{$k \leftarrow \Bigl\lfloor B/d/b \Bigr\rfloor$\;
$\opt_\infty(k) \leftarrow g(k)$ 
by the output of Algorithm~\ref{Alg:greedy k-center}\;
$\Delta(b) \leftarrow 2^{-(b-1-m_e)}\max_{\ybf_i\in \mathcal{Y}}\|\ybf_i\|$\;
$\epsilon(k, b) \leftarrow \rho \cdot \opt_\infty(k) + \rho \Delta(b) + \rho^2 \cdot \opt_\infty(k) \Delta(b)$\;
}
$(k^*, b^*)\leftarrow \argmin \epsilon(k,b)$\;
\textbf{return} $(k^*, b^*)$\;
\end{algorithm}
\normalsize
\setlength{\textfloatsep}{.5em}

\subsection{Discussions} 
\label{subsec: Discussion on HM-RCC and GM-RCC}
\subsubsection{Performance Comparison}
The straightforward solution EM (Section~\ref{subsubsec:baselineEM}) directly minimizes the error bound (\ref{MD: obj}) and is thus expected to find the best configuration for CS + QT. In comparison, the two proposed algorithms (EVD-MECB and MD-MECB) only optimize approximations of the error bound. It is difficult to theoretically analyze or compare the ML errors of  these algorithms since the bound may be loose and the approximations may be smaller than the bound. 
Instead, we will use empirical evaluations to compare the actual ML errors achieved by these algorithms (see Section~\ref{sec: Performance Evaluation}).

\subsubsection{Complexity Comparison}
In terms of complexity, EM involves executions of the $k$-means algorithm for all $(k,b)$ pairs, which is thus computationally complicated. In comparison, EVD-MECB only requires one eigenvalue decomposition (EVD) and one matrix multiplication,  while MD-MECB only needs to invoke Algorithm~\ref{Alg:greedy k-center} once. 
Therefore, both of them can be implemented efficiently. As EVD can be computed with complexity $O(n^3)$ \cite{demmel2007fast}, EVD-MECB has a complexity of $O(n^3 + d + b_0)$. Since the computational complexity of  Algorithm~\ref{Alg:greedy k-center} is $O(n^2)$ (achieved at $K=n$), MD-MECB has a complexity of $O(n^2 + b_0)$. 
Hence, MD-MECB is expected to be more efficient than EVD-MECB, which will be further validated in Section~\ref{sec: Performance Evaluation}.


\section{Distributed Setting}\label{sec: distributed setting }
We now describe how to construct a quantized coreset under a global communication budget in distributed settings. Suppose that the data are distributed over $N$ nodes as $\{\mathcal{Y}_1,\ldots,\mathcal{Y}_N\}$. Our goal is to configure the construction of local coresets $\{\mathcal{S}_1,\ldots,\mathcal{S}_N\}$ such that $\bigcup_{i=1}^N \mathcal{S}_i$ can be represented by no more than $B$ bits and is an $\epsilon$-coreset for $\bigcup_{i=1}^N \mathcal{Y}_i$ for the smallest $\epsilon$. Given a distribution of the budget to each node, we can use the algorithms in Section~\ref{sec: Two Effective Heuristic Solutions} to make each $\mathcal{S}_i$ an $\epsilon_i$-coreset of the local dataset $\mathcal{Y}_i$ for the smallest $\epsilon_i$. However, the following questions remain: (1) How is $\epsilon$ related to $\epsilon_i$'s? (2) How can we distribute the global budget $B$ to minimize $\epsilon$?
In this section, we answer these questions by formulating and solving the distributed version of the MECB problem. \looseness=-1



\subsection{Problem Formulation in Distributed Setting}\label{subsec: Problem formulation in distributed setting}

In the following, we first show that $\epsilon = \max_i\epsilon_i$, and then formulate the MECB problem in the distributed setting. 

\begin{lemma}\label{lem: coreset of merged coreset}
If $\mathcal{C}_1$ and $\mathcal{C}_2$ are $\epsilon_1$-coreset and $\epsilon_2$-coreset for datasets $\mathcal{Y}_1$ and $\mathcal{Y}_2$, respectively, w.r.t. a cost function, then $\mathcal{C}_1 \bigcup \mathcal{C}_2$ is a $\max\{\epsilon_1, \epsilon_2\}$-coreset for $\mathcal{Y}_1 \bigcup \mathcal{Y}_2$ w.r.t. the same cost function. 
\end{lemma}
\begin{proof}
We consider both sum and maximum cost functions. Without loss of generality, we assume $\epsilon_2 \geq \epsilon_1$.

\emph{Sum cost:} Given a feasible solution $\xbf$, we consider sum cost as $\cost(\mathcal{Y}, \xbf)= \sum_{\ybf\in \mathcal{Y}} \cost(\ybf, \xbf)$. 
According to Definition \ref{def:epsilon-coreset}, we have $ (1 - \epsilon_1) \sum_{\ybf \in \mathcal{Y}_1} \cost(\ybf, \xbf) \leq \sum_{\mathbf{c} \in \mathcal{C}_1} \cost(\mathbf{c}, \xbf) \leq (1 + \epsilon_1) \sum_{\ybf \in \mathcal{Y}_1} \cost(\ybf, \xbf)$ and $ (1 - \epsilon_2) \sum_{\ybf \in \mathcal{Y}_2} \cost(\ybf, \xbf) \leq \sum_{\mathbf{c} \in \mathcal{C}_2} \cost(\mathbf{c}, \xbf) \leq (1 + \epsilon_2) \sum_{\ybf \in \mathcal{Y}_2} \cost(\ybf, \xbf)$. 
Summing up these two equations and noting that $\epsilon_2 \geq \epsilon_1$, we can conclude that $\mathcal{C}_1 \bigcup \mathcal{C}_2$ is an $\epsilon_2$-coreset for $\mathcal{Y}_1 \bigcup \mathcal{Y}_2$.

\emph{Maximum cost:} The proof for maximum cost function is similar as above but taking the maximum instead. 
\end{proof}

We can easily extend Lemma VI.1 to multiple nodes to compute the global $\epsilon$ error as: $\epsilon = \max_i\epsilon_i$. Thus the objective of minimizing $\epsilon$ is equivalent to minimizing the largest $\epsilon_i$ for $i\in\{1,\ldots,N\}$. 

Let $B_i$ denote the local budget for the $i$-th node. Intuitively, the larger the local budget $B_i$, the smaller 
$\epsilon_i$. 
Therefore, we model $\epsilon_i$ as a non-increasing function w.r.t. the local budget $B_i$, denoted by $\epsilon_i(B_i)$.  
\begin{figure}[t!]
   \centerline{\includegraphics[width=.5\linewidth]{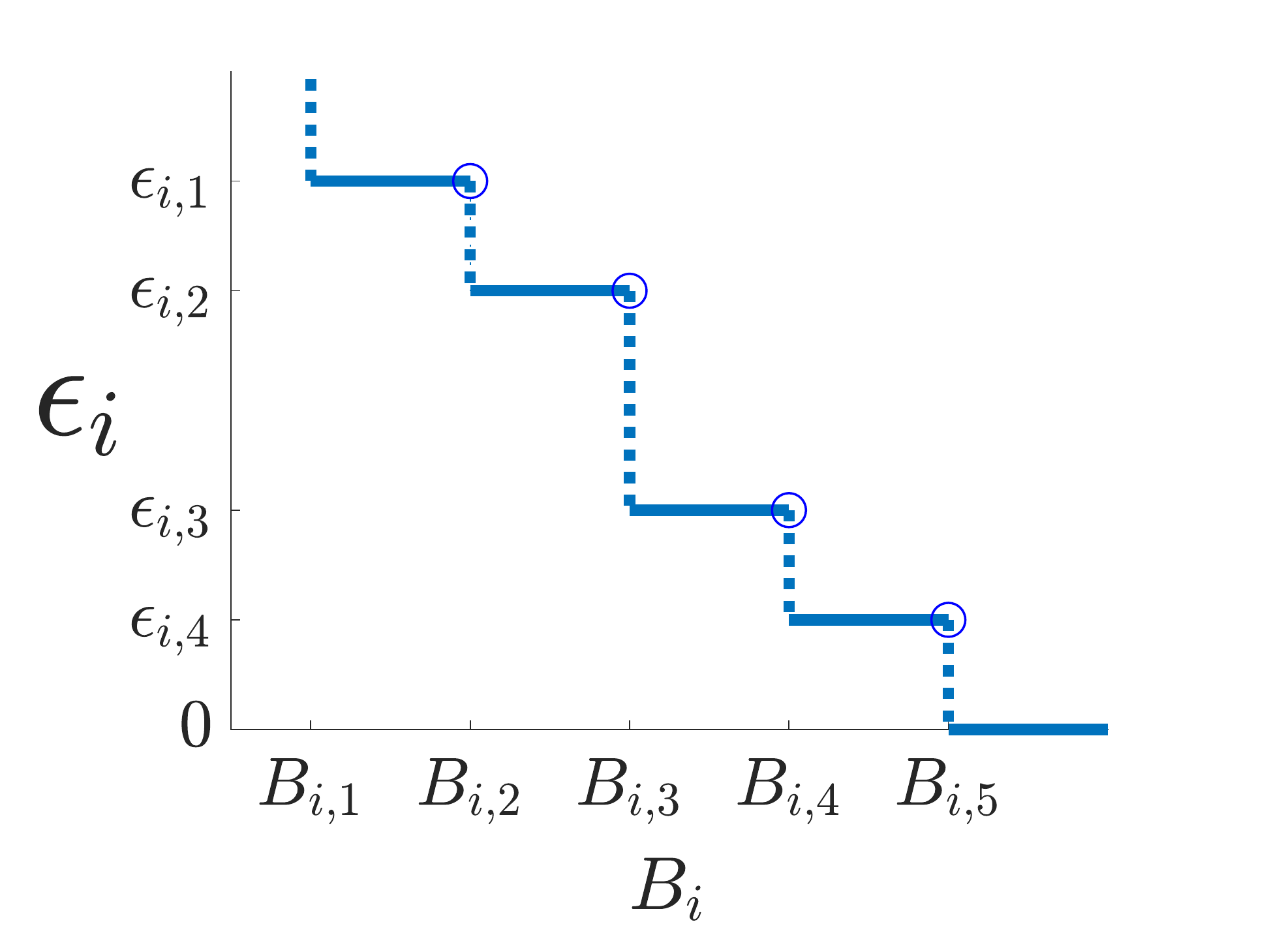}}
     \vspace{-.5em}
    \caption{Illustration of step objective functions. }
    \label{fig:step_func}
     \vspace{-.5em}
\end{figure}

Then, we can formulate the MECB problem in the distributed setting (MECBD) as follows:
\begin{subequations}\label{eq:distributed formulation}
\begin{align}
\min\quad &{\max_{i\in \{1,\ldots,N\}}}  \epsilon_i(B_i) \label{distributed: obj}\\ 
\textrm{s.t.} \quad & \sum_{i=1}^N B_i \leq B. \label{distributed: constraint}
\end{align} 
\end{subequations}
Note that to compute $\epsilon_i(B_i)$ for a given $B_i$, we need to solve an instance of the MECB problem in (\ref{prob: concrete}) for dataset $\mathcal{Y}_i$ and budget $B_i$.

\begin{algorithm}[tb]
\caption{OBA-MECBD}
\label{Alg:OBA-RCC}
\small
\SetKwInOut{Input}{input}\SetKwInOut{Output}{output}
\SetKwFor{EachNode}{each node $i = 1, ..., N$:}{}{}
\SetKwFor{TheServer}{the server:}{}{}
\Input{Distributed datasets $\{\mathcal{Y}_i\}_{i=1}^N$, Lipschitz constant $\rho$, communication budget $B$. }
\Output{Optimal $\{(k^*_i, b^*_i)\}_{i=1}^N$ to configure the construction of local quantized coresets within a global budget $B$.}
\BlankLine

\nonl \EachNode{}
{
    $B_{0} \leftarrow 1\cdot (1+m_e)\cdot d$\; \label{line: OBA-step1-start}
    compute $\epsilon_i(B_0)$ by MD-MECB or EVD-MECB\; \label{line: OBA-subroutine1} 
    $\mathcal{E}_i \leftarrow \{(B_0, \epsilon_i(B_0))\}$\;
    \ForEach{integer $B_i \in [B_0+1, |\mathcal{Y}_i|\cdot b_0\cdot d] $}
    {
        compute $\epsilon_i(B_i)$ by MD-MECB or EVD-MECB\;\label{line: step1 MD/EVD}
        \If{$\epsilon_i(B_i) < \min\{\epsilon_{i,j}:\: (B_{i,j}, \epsilon_{i,j})\in \mathcal{E}_i\} $ }
        {
            $\mathcal{E}_i \leftarrow \mathcal{E}_i \cup \{(B_i, \epsilon_i(B_i))\}$\;
        }
    }

    
    report $\mathcal{E}_i$ to the server;  \label{line: OBA-step1-end}
}
 \nonl \TheServer{}
{
    $\mathbf{E}$ is an ordered list of $\epsilon$-values in $\bigcup_i \mathcal{E}_i$, sorted in descending order\;\label{line: OBA-step2-start}
    $I_{max} \leftarrow $ first index in $\mathbf{E}$\; 
    $I_{min} \leftarrow $ last index in $\mathbf{E}$\; 
    \While{true \label{line: OBA-step2-start while}}
    {
        $I \leftarrow \Bigl\lfloor \frac{I_{max} + I_{min}}{2}\Bigr\rfloor$\;
        $\epsilon_I \leftarrow $ the $I$-th element in $\mathbf{E}$\; 
        \For{$i = 1, \ldots, N$ \label{line: OBA-step2-for}}
        {
            $B_i(\epsilon_I)\leftarrow \min\{B_{i,j}:\: (B_{i,j}, \epsilon_{i,j})\in \mathcal{E}_i, \: \epsilon_{i,j}\leq \epsilon_I\}$\; 
        }\label{line: OBA-step2-forend}
        \eIf{$I_{min} = I_{max} + 1$ \label{line: OBA-stopping-condition}}
        {
            send $B_i(\epsilon_I)$ to node $i$ for each $i=1,\ldots,N$\;
            break \textbf{while} loop\;
        }
        {
            \eIf{$\sum_{i=1}^N B_i(\epsilon_I) > B$\label{line: OBA-binary-begin}}
            {
                $I_{min} = I$\;
            }
            {
                $I_{max} = I$\;
            }\label{line: OBA-binary-end}
        }
    }\label{line: OBA-step2-end}
}
 \nonl \EachNode{}
{
    find local $(k_i^*, b_i^*)$ under budget $B_i(\epsilon_I)$ given by the server by MD-MECB or EVD-MECB\;\label{line: OBA-step3}
}
\textbf{return} $\{(k^*_i, b^*_i)\}_{i=1}^N$\;\label{line: OBA-step5}
\end{algorithm}
\normalsize
\setlength{\textfloatsep}{.5em}

\subsection{Optimal Budget Allocation Algorithm for MECBD (OBA-MECBD)}\label{subsec: MECBD Algorithm}
The MECBD problem in (\ref{eq:distributed formulation}) is a \emph{minimax knapsack problem} \cite{luss1991nonlinear, luss1987algorithm} with a nonlinear non-increasing objective function. Special cases of this problem with strictly decreasing objective functions have been solved in \cite{luss1987algorithm}. However, the objective function of MECBD is a step function as shown below, which is not strictly decreasing. 
Below we will develop a polynomial-time algorithm to solve our instance of the minimax knapsack problem using the following property of $\epsilon_i(B_i)$.\looseness=-1

We note that $\epsilon_i(B_i)$ is a non-increasing step function of $B_i$ (see Figure~\ref{fig:step_func}). This is because the configuration parameters $k$ and $b$ in the CS + QT procedure are integers. 
Therefore, there exist intervals $[B_{i,j}, B_{i,j+1})$ ($j = 0,1,2,...$) such that for any $B_i, B'_i \in [B_{i,j}, B_{i,j+1})$, we have $\epsilon_i(B_i) = \epsilon_i(B'_i)$, as shown in  Figure~\ref{fig:step_func}.
Given a target value of $\epsilon_i$, the minimum $B_i$ for reaching this target is thus always within the set $\{B_{i,j}\}$. \looseness=-1

Our algorithm, shown in Algorithm~\ref{Alg:OBA-RCC}, has three main steps. First, in lines~\ref{line: OBA-step1-start}--\ref{line: OBA-step1-end}, each node computes the set $\mathcal{E}_i$ of all pairs of $B_{i,j}$ and the corresponding $\epsilon_i(B_{i,j})$. This is achieved by evaluating $\epsilon_i(B_i)$ according to MD-MECB or EVD-MECB for gradually increasing $B_i$ and recording all the points where $\epsilon_i(\cdot)$ decreases. 
After that, the set $\mathcal{E}_i$ is sent to a server.

Second, the server allocates the global budget to the nodes according to lines~\ref{line: OBA-step2-start}--\ref{line: OBA-step2-end}. 
To this end, it computes an ordered list $\mathbf{E}$ of all possible values of the global $\epsilon := \max_i \epsilon_i$.
Let $B_i(\epsilon)$ denote the smallest value of $B_i$ such that $\epsilon_i(B_i) \leq \epsilon$. 
The main idea is to perform a binary search for the target value of $\epsilon \in \mathbf{E}$ (lines~\ref{line: OBA-step2-start while}--\ref{line: OBA-step2-end}). 
For each candidate value of $\epsilon$, we compute $B_i(\epsilon)$ for all $i$. If $\sum_i B_i(\epsilon) < B$ (i.e., we are below the budget when targeting at the current choice of $\epsilon$), we will eliminate all $\epsilon'\in\mathbf{E}$  such that $\epsilon' > \epsilon$; otherwise, we will eliminate all $\epsilon'\in\mathbf{E}$  such that $\epsilon' < \epsilon$. 
After finding the target value of $\epsilon$ such that $\sum_i B_i(\epsilon)$ achieves the largest value within $B$, the server sends the corresponding local budget $B_i(\epsilon)$ to each node. \looseness=-1

Finally, each node uses MD-MECB or EVD-MECB to compute its local configuration $(k_i^*, b_i^*)$ under the given budget. \looseness=-1

\emph{Complexity:} We analyze the complexity step by step. First, computing $\mathcal{E}_i$ at each node (lines~\ref{line: OBA-step1-start}--\ref{line: OBA-step1-end}) has a complexity of $O((n^2 + b_0)db_0 n)$ if using MD-MECB and $O((n^3 + d + b_0)db_0 n)$ if using EVD-MECB, dominated by line~\ref{line: step1 MD/EVD}. 
Second, computing the budget allocation at the server (lines~\ref{line: OBA-step2-start}--\ref{line: OBA-step2-end}) has a complexity of $O(db_0 n \log(db_0 n))$. Specifically, as $\mathbf{E}$ has $O(db_0 n)$ elements, sorting it takes $O(db_0 n \log(db_0 n))$. The \textbf{while} loop is repeated $O(\log(db_0 n))$ times, as each loop eliminates half of the candidate $\epsilon$ values in $\mathbf{E}$, and each loop takes $O(db_0 n)$, dominated by lines~\ref{line: OBA-step2-for}--\ref{line: OBA-step2-forend}. Finally, computing the local configuration at each node (line~\ref{line: OBA-step3}) takes $O(n^2 + b_0)$ using MD-MECB and $O(n^3 + d + b_0)$ using EVD-MECB.






\emph{{Optimality: }}
{Next, we prove the optimality of OBA-MECBD in budget allocation.} Let $\epsilon_i^\pi(B_i)$ be the error bound for a given solution $\pi$ of the MECB problem for dataset $\mathcal{Y}_i$ and budget $B_i$. We show that OBA-MECBD is optimal in the following sense. \looseness=-1

\begin{theorem}\label{thm:OBA-MECBD optimality}
Using a given MECB algorithm $\pi$ as the subroutine called in lines~\ref{line: OBA-subroutine1}, \ref{line: step1 MD/EVD}, and \ref{line: OBA-step3}, OBA-MECBD solves MECBD optimally w.r.t. $\pi$, i.e., its budget allocation $(B_i)_{i=1}^N$ is the optimal solution to (\ref{eq:distributed formulation}) with $\epsilon_i(B_i)$ replaced by $\epsilon_i^\pi(B_i)$.
\end{theorem}
\begin{proof}
Let $B_i^\pi(\epsilon)$ denote the smallest value of $B_i$ such that $\epsilon_i^\pi(B_i)\leq \epsilon$. 
By lines~\ref{line: OBA-binary-begin}--\ref{line: OBA-binary-end} in Algorithm~\ref{Alg:OBA-RCC}, $I_{min}$ and $I_{max}$ should always satisfy $\sum_{i=1}^N B_i^\pi(\epsilon_{I_{min}}) > B$ and $\sum_{i=1}^N B_i^\pi(\epsilon_{I_{max}}) \leq B$ for all nontrivial values of $B$. Let $\epsilon^*$ denote the value of $\epsilon_I$ at the end of budget allocation, which is the value of (\ref{distributed: obj}) achieved by OBA-MECBD.
As $I = I_{max}$ and $I_{min} = I_{max}+1$ at this time, $\epsilon^*$ must be the smallest value of $\epsilon$ such that  $\sum_{i=1}^N B_i^\pi(\epsilon) \leq B$. Therefore, for any other budget allocation $(B'_i)_{i=1}^N$ such that $\sum_{i=1}^N B'_i \leq B$, we must have $\max_i \epsilon_i^\pi(B'_i) \geq \epsilon^*$.
\end{proof}

\begin{figure}[t]
\begin{subfigure}{0.25\textwidth}
\centerline{
\includegraphics[width=\textwidth,height=3.2cm]{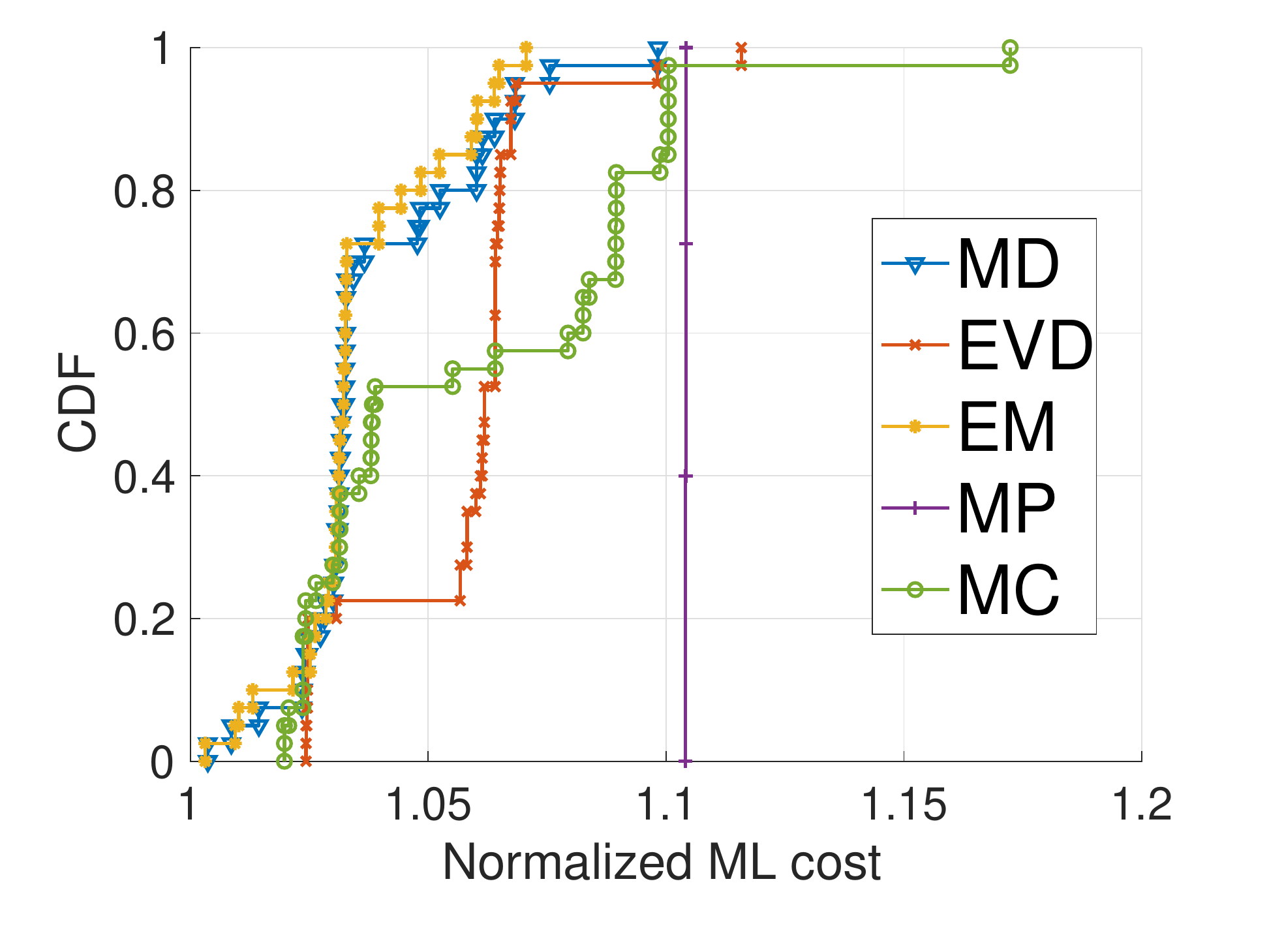}}
\vspace{-.1em}
\caption{MEB}
\end{subfigure}
 \hspace{-6em}
\begin{subfigure}{0.25\textwidth}
\centerline{
\includegraphics[width=\textwidth,height=3.2cm]{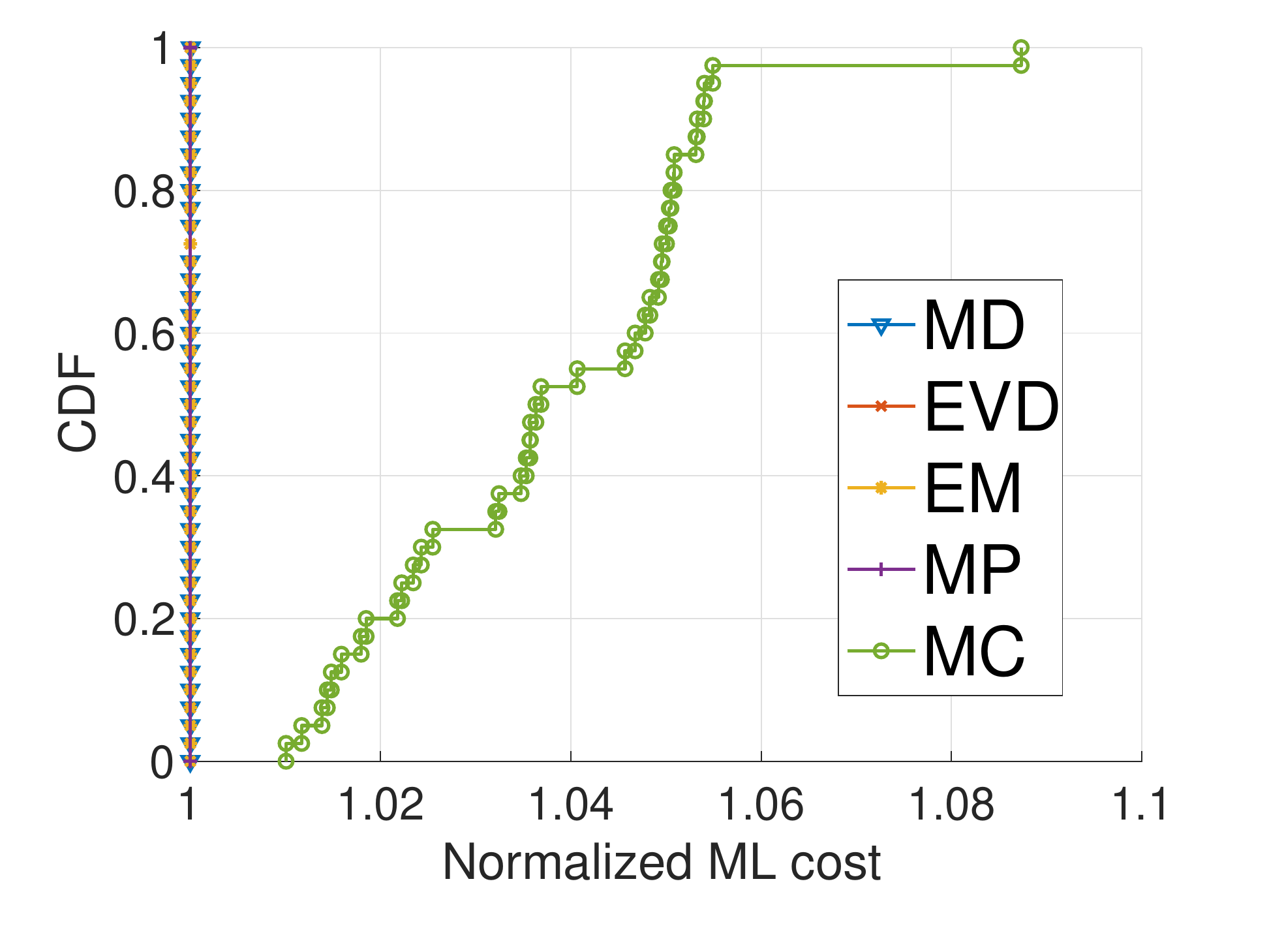}}
\vspace{-.1em}
\caption{$k$-means ($k=2$)}
\end{subfigure}
  \begin{subfigure}{.25\textwidth}
  \centerline{
   \includegraphics[width=\textwidth,,height=3.2cm]{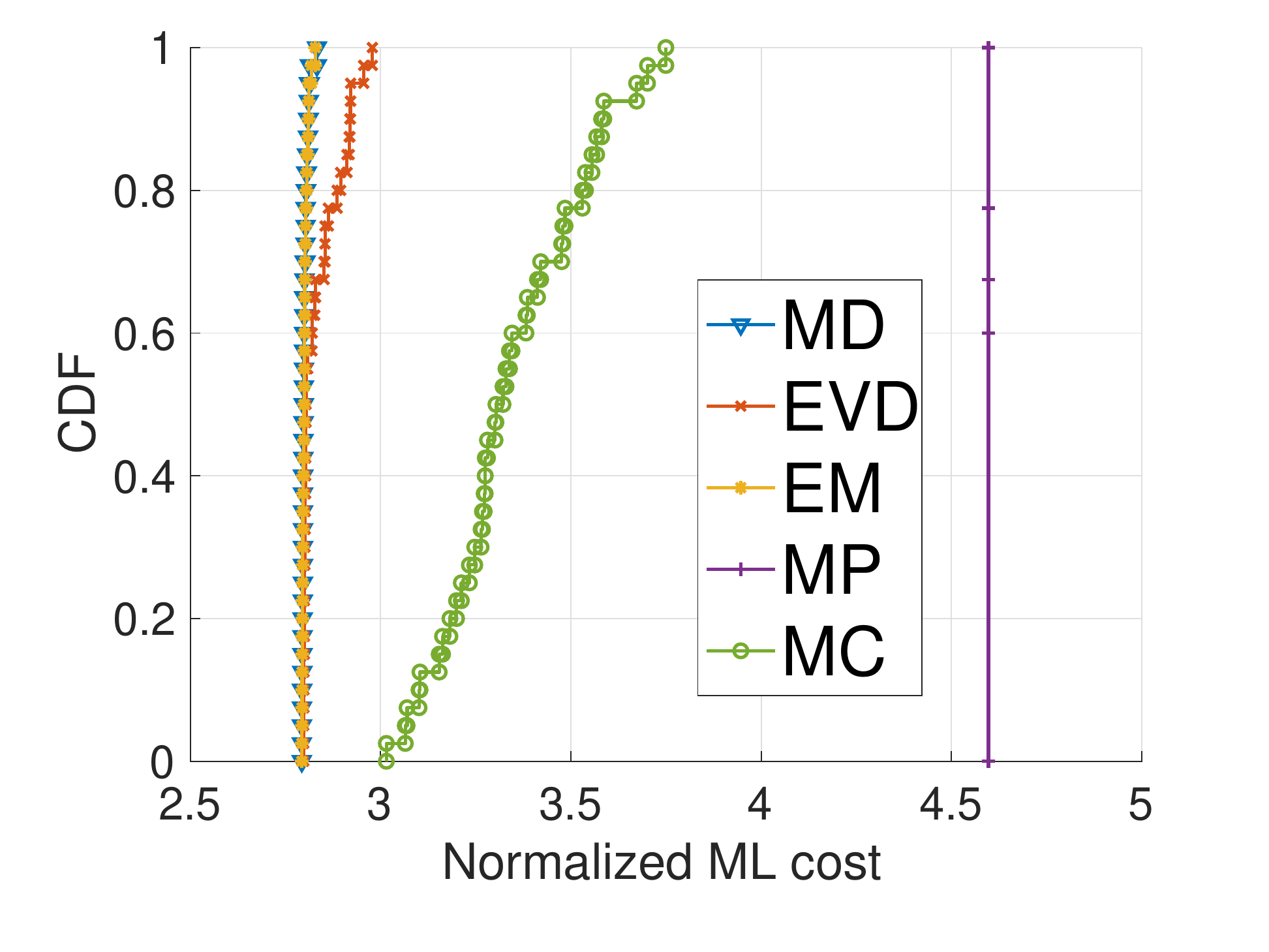}}
   \vspace{-.1em}
    \caption{PCA (3 components) }
  \end{subfigure}
       \hspace{-1.5em}
  \begin{subfigure}{0.25\textwidth}
    \centerline{
  \includegraphics[width=\textwidth,height=3.2cm]{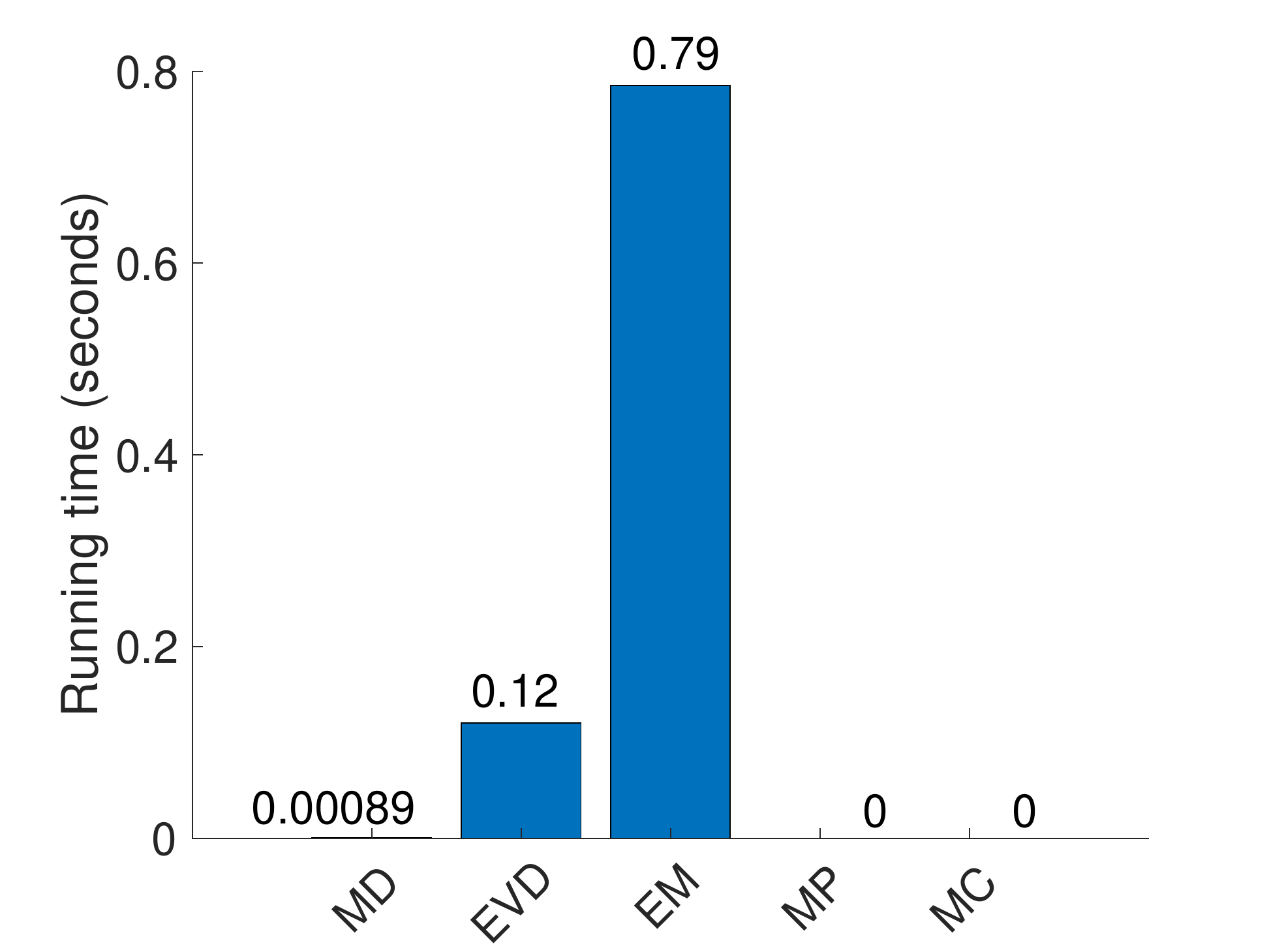}}
  \vspace{-.1em}
\caption{Running time}   
   \end{subfigure}    
\caption{Evaluation on Fisher's Iris dataset (centralized setting). 
}
\label{fig:fisheriris centralized}
   \vspace{1em}
\end{figure}

\section{Performance Evaluation}\label{sec: Performance Evaluation}
In this section, we evaluate our proposed algorithms using multiple real-world datasets for various ML tasks. Our objective is to validate the effectiveness and efficiency of our proposed algorithms (EVD-MECB, MD-MECB, OBA-MECBD) against benchmarks.

\subsection{Datasets}
In our experiments, we use four real-world datasets to evaluate our algorithms: (1) Fisher's Iris dataset \cite{FisherIris}, with 3 classes, 50 data points in each class, 5 attributes for each data point; (2) Facebook  metric dataset \cite{FacebookMetrics}, which has 494 data points with 19 attributes; (3) Pendigits dataset \cite{Pendigits}, which has $7,494$ data points and 17 attributes; 
(4) MNIST handwritten digits dataset in a 784-dimensional space~\cite{MNIST}, where we use $60,000$ data points for training and $10,000$ data points for testing. We leverage the approach in \cite{coreset19:report} to pre-process the labels, i.e., each label is mapped to a number such that distance between points with the same label is smaller than distance between points with different labels. All the original data are represented in the standard IEEE 754 double-precision binary floating-point format \cite{IEEE754}.

\subsection{ML Tasks}
We consider four ML tasks: (1) minimum enclosing ball (MEB) \cite{clarkson2010coresets}; (2) $k$-means ($k=2$ in our experiments); (3) principal component analysis (PCA); and (4) neural network (NN) (with three layers, 100 neurons in the hidden layer). Tasks (1--3) are unsupervised, and task (4) is supervised.

\subsection{Algorithms}
For the centralized setting, we consider five different algorithms for comparison. The first two are the proposed algorithms, i.e.,  EVD-MECB in Algorithm~\ref{Alg:heuristic kmeans coreset} (denoted by \emph{EVD}), MD-MECB in Algorithm~\ref{Alg:greedy max-distance coreset} (denoted by \emph{MD}). The third algorithm is the straightforward solution \emph{EM} (see Section~\ref{subsubsec:baselineEM}). The fourth algorithm aims to Maximize the Precision (\emph{MP}), i.e., using the configuration $k=\Bigl\lfloor \frac{B}{d\cdot b_0}\Bigr\rfloor$ and $b= b_0$ to construct coresets. The fifth algorithm aims to Maximize the Cardinality (\emph{MC}), i.e., using $k=\min(n, \Bigl\lfloor \frac{B}{d \cdot (1+m_e)}\Bigr\rfloor)$ and $b = \max(1+m_e, \Bigl\lfloor \frac{B}{d \cdot n}\Bigr\rfloor)$ to construct coresets, where $1+m_e$ is the minimum number of bits required to represent a number by the rounding-based quantizer (Section~\ref{subsec: quantizer analysis}). 

For the distributed setting, we consider six algorithms for comparison. The first five algorithms correspond to instances of OBA-MECBD in Algorithm~\ref{Alg:OBA-RCC} that use EVD-MECB, MD-MECB, EM, MP, and MC as their subroutines, respectively. We denote these algorithms by OBA-EVD, OBA-MD, OBA-EM, OBA-MP and OBA-MC, respectively. The sixth algorithm is DRCC as proposed in \cite{coreset19:report} that optimizes the allocation of a given coreset cardinality to individual nodes.

\begin{figure}[t]
\begin{subfigure}{0.25\textwidth}
\centerline{
\includegraphics[width=\textwidth,height=3.2cm]{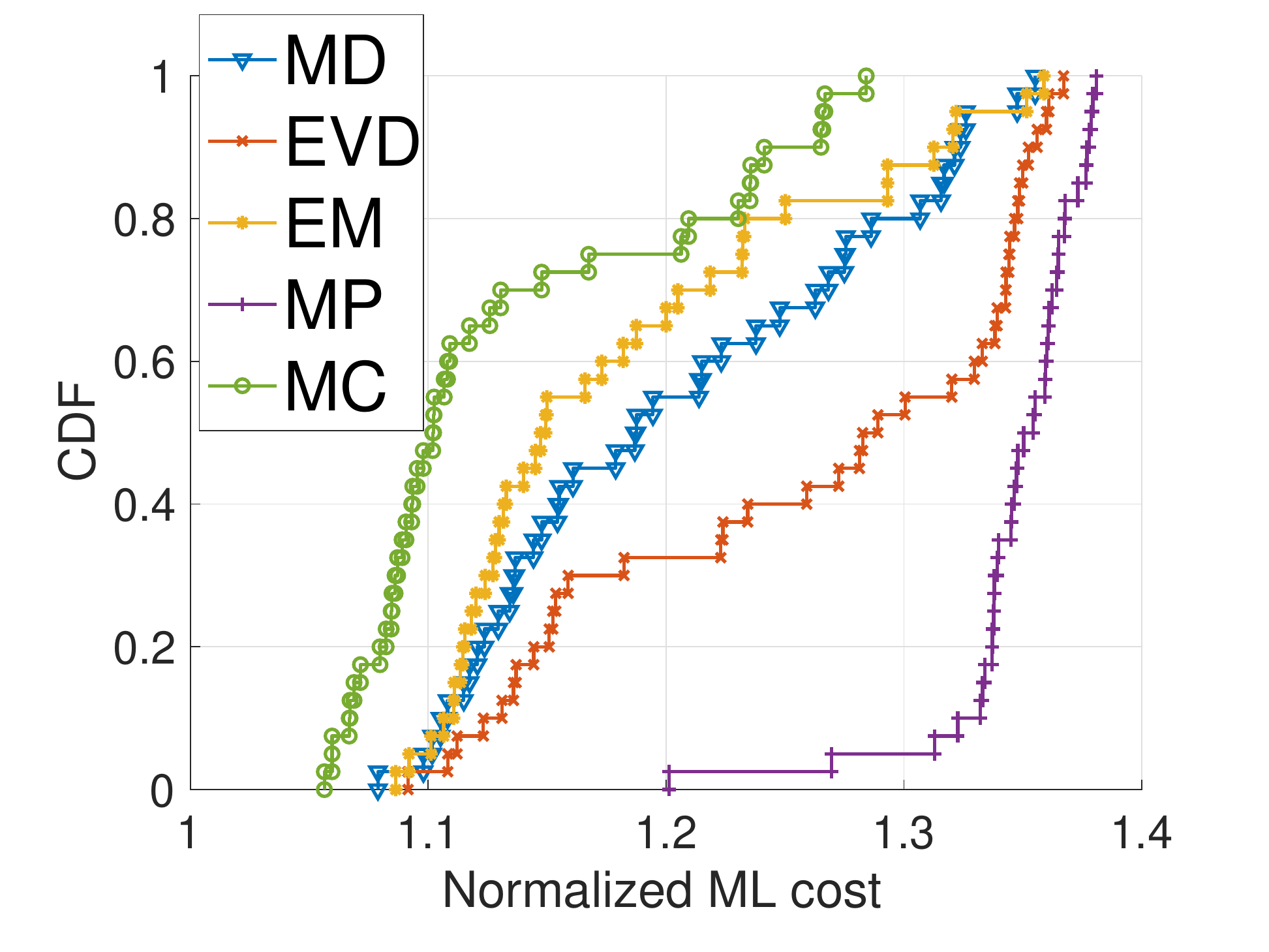}}
\vspace{-.1em}
\caption{MEB}
\end{subfigure}
 \hspace{-6em}
\begin{subfigure}{0.25\textwidth}
\centerline{
\includegraphics[width=\textwidth,height=3.2cm]{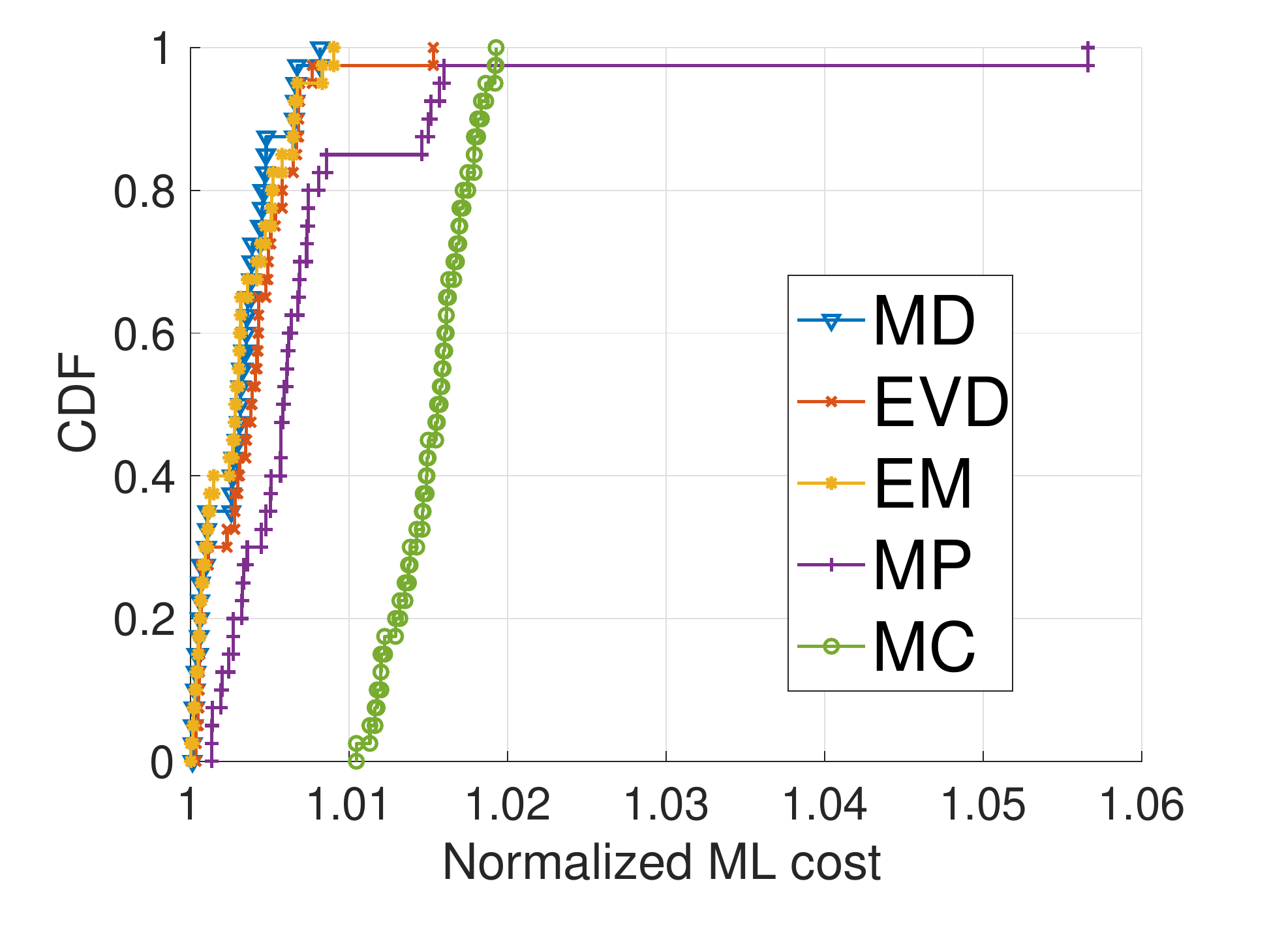}}
\vspace{-.1em}
\caption{$k$-means ($k=2$)}
\end{subfigure}
  \begin{subfigure}{.25\textwidth}
  \centerline{
   \includegraphics[width=\textwidth,,height=3.2cm]{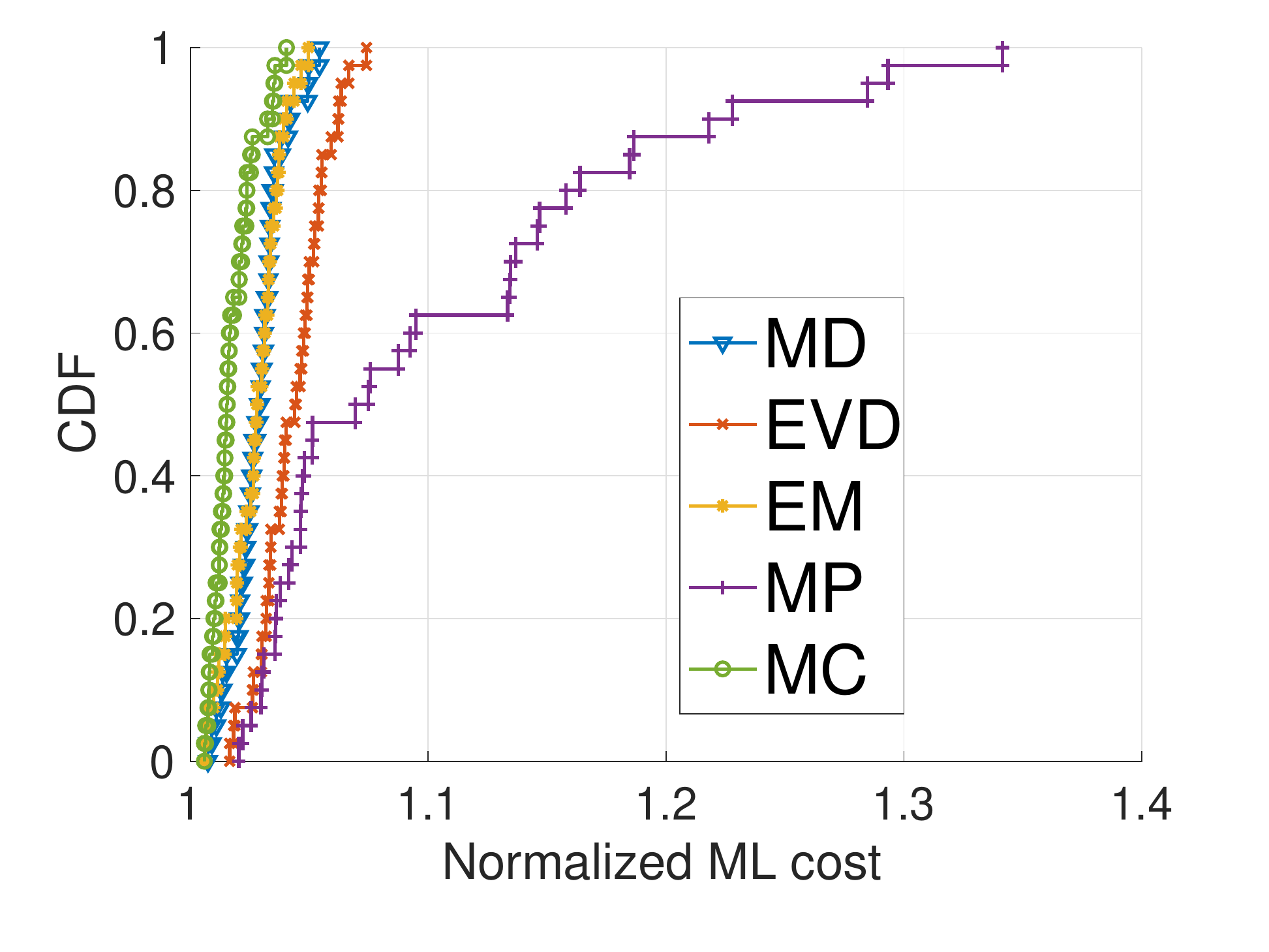}}
   \vspace{-.1em}
    \caption{PCA (3 components) }
  \end{subfigure}
   \hspace{-1.5em}
  \begin{subfigure}{0.25\textwidth}
    \centerline{
  \includegraphics[width=\textwidth,height=3.2cm]{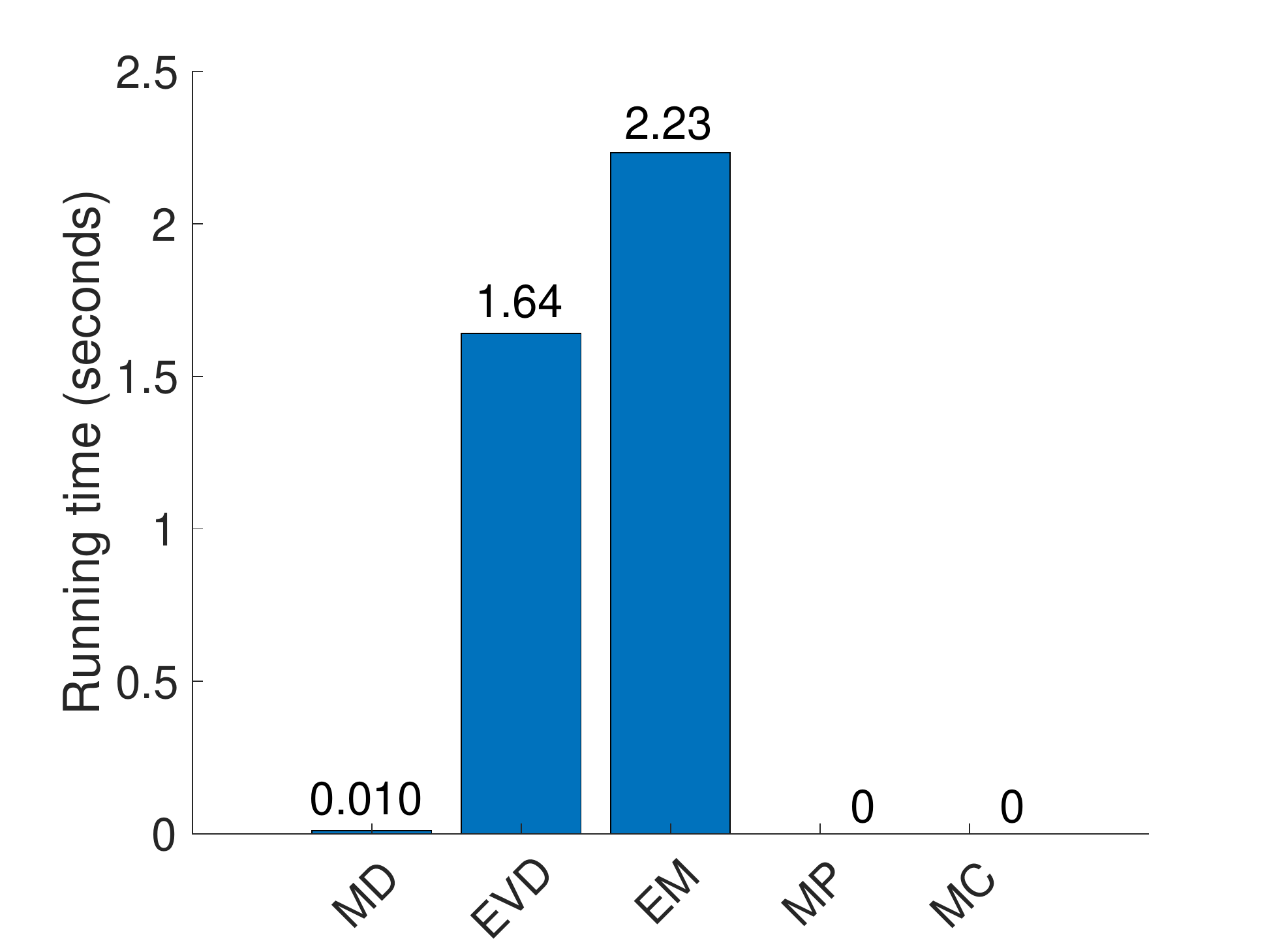}}
  \vspace{-.1em}
\caption{Running time}   
   \end{subfigure}    
\caption{Evaluation on Facebook metric dataset (centralized setting). }
\label{fig:facebook centralized}
   \vspace{1em}
\end{figure}

\subsection{Performance Metrics}
We use the \emph{normalized ML cost} to measure the performance of unsupervised ML tasks. The normalized ML cost is defined as $\cost(\mathcal{Y}, \xbf_{\mathcal{S}})/\cost(\mathcal{Y}, \xbf^*)$, where $\xbf_\mathcal{S}$ is the model learned from coreset $\mathcal{S}$ and $\xbf^*$ is the model learned from the original dataset $\mathcal{Y}$. For supervised ML tasks, we use \emph{classification accuracy} to measure the performance. Furthermore, we report the running time of each algorithm. 
All metrics are computed over $40$ Monte Carlo runs unless stated otherwise. 

\subsection{Results in Centralized Setting}
\subsubsection{Unsupervised Learning}
We first evaluate the unsupervised learning tasks: MEB, $k$-means, and PCA. We perform this evaluation on three datasets: 
Fisher's Iris, Facebook metric, and Pendigits. Figures~\ref{fig:fisheriris centralized}--\ref{fig:pendigits centralized} show the cumulative distribution function (CDF) of normalized ML costs as well as the average running time of each algorithm, when the budget is set to $2\%$ of the size of the original dataset, i.e., $B = 960$, $12014$ and $163069$, respectively. We also list the $b^*$ values over the Monte Carlo runs for EVD-MECB (\emph{EVD}), MD-MECB (\emph{MD}), and EM in Tables~\ref{tab: fisher iris}--\ref{tab: Pendigits}. We have the following observations: 
    1) In most cases, our proposed algorithms EVD-MECB and MD-MECB yield coresets that are much smaller ($98\%$ smaller) than the original dataset but support these ML tasks with less than $10\%$ degradation in performance. 
    2) Compared to the proposed algorithms, EM achieves a slightly better ML performance, but has a much higher running time. 
    3) Compared to EVD-MECB, MD-MECB is not only faster, but also more closely approximates EM. 
    4) Compared with MP and MC that relies on a single operation, the algorithms jointly optimizing coreset constrution and quantization (EM, EVD-MECB, MD-MECB) achieve much better ML performance over all.

\begin{figure}[t]
\begin{subfigure}{0.25\textwidth}
    \centerline{
    \includegraphics[width=\textwidth,height=3.2cm]{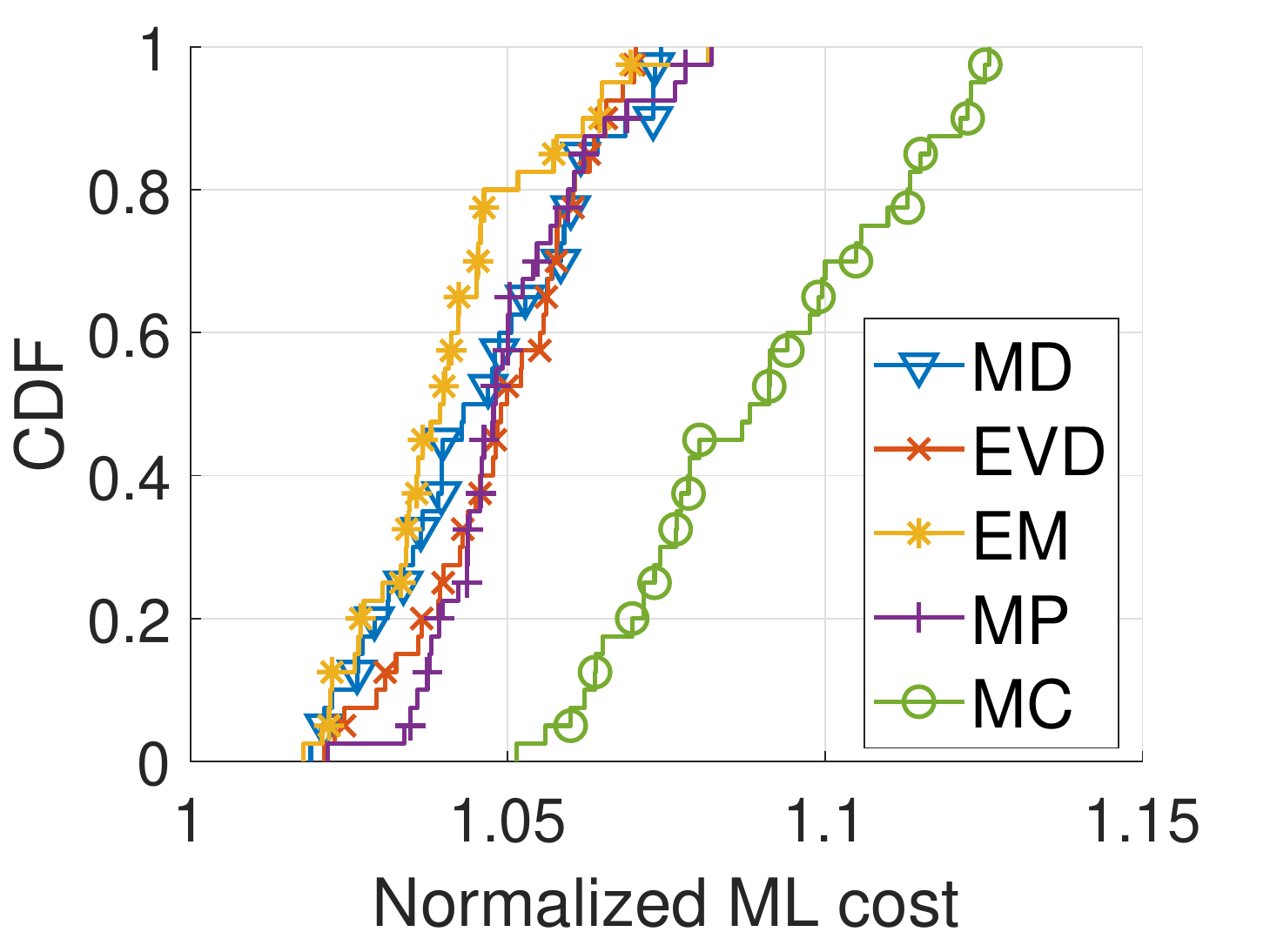}}
    \vspace{-.1em}
    \caption{MEB}
\end{subfigure}
 \hspace{-6em}
\begin{subfigure}{0.25\textwidth}
    \centerline{
    \includegraphics[width=\textwidth,height=3.2cm]{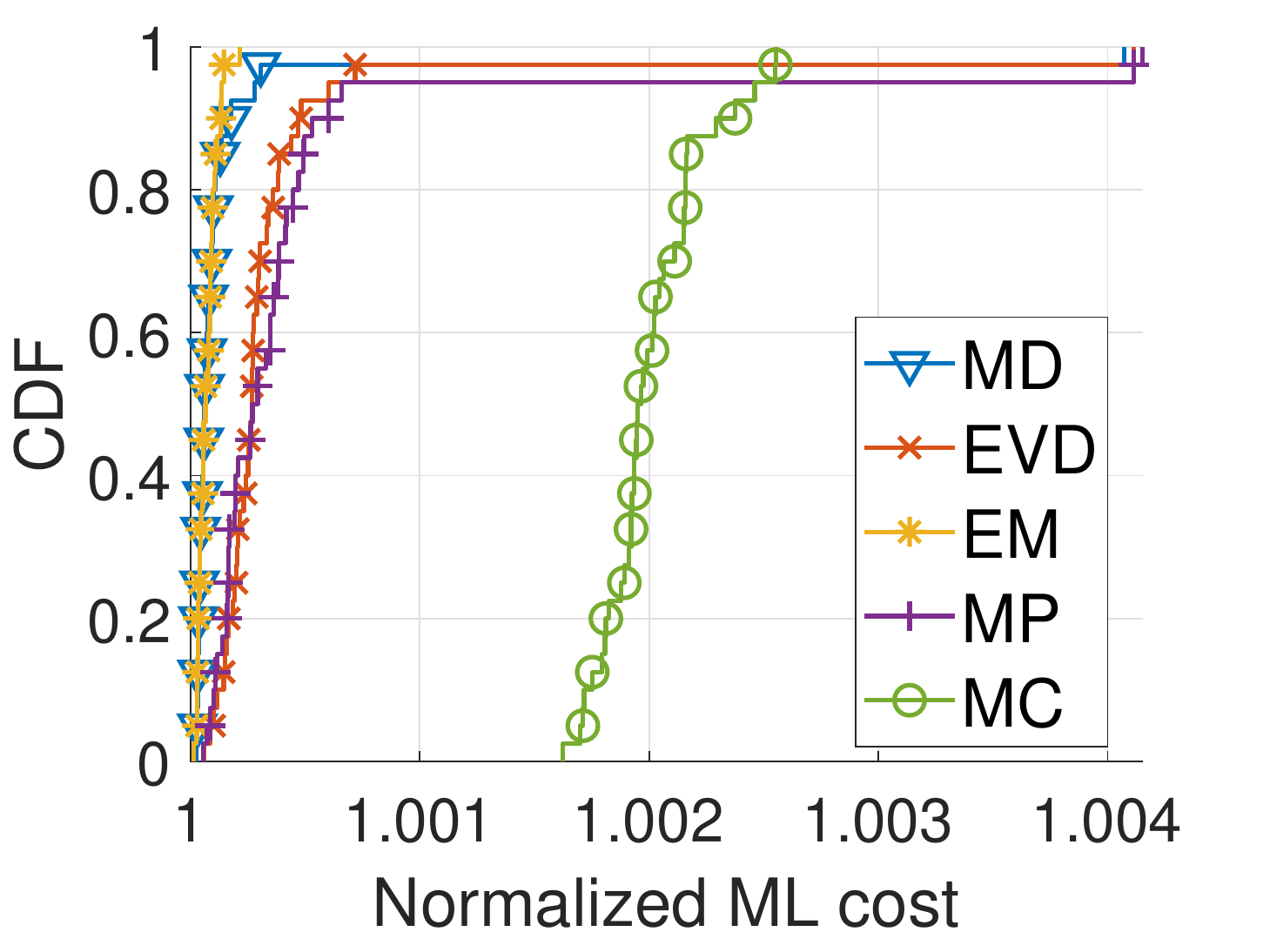}}
    \vspace{-.1em}
    \caption{$k$-means ($k=2$)}
\end{subfigure}
\begin{subfigure}{.25\textwidth}
    \centerline{
    \includegraphics[width=\textwidth,,height=3.2cm]{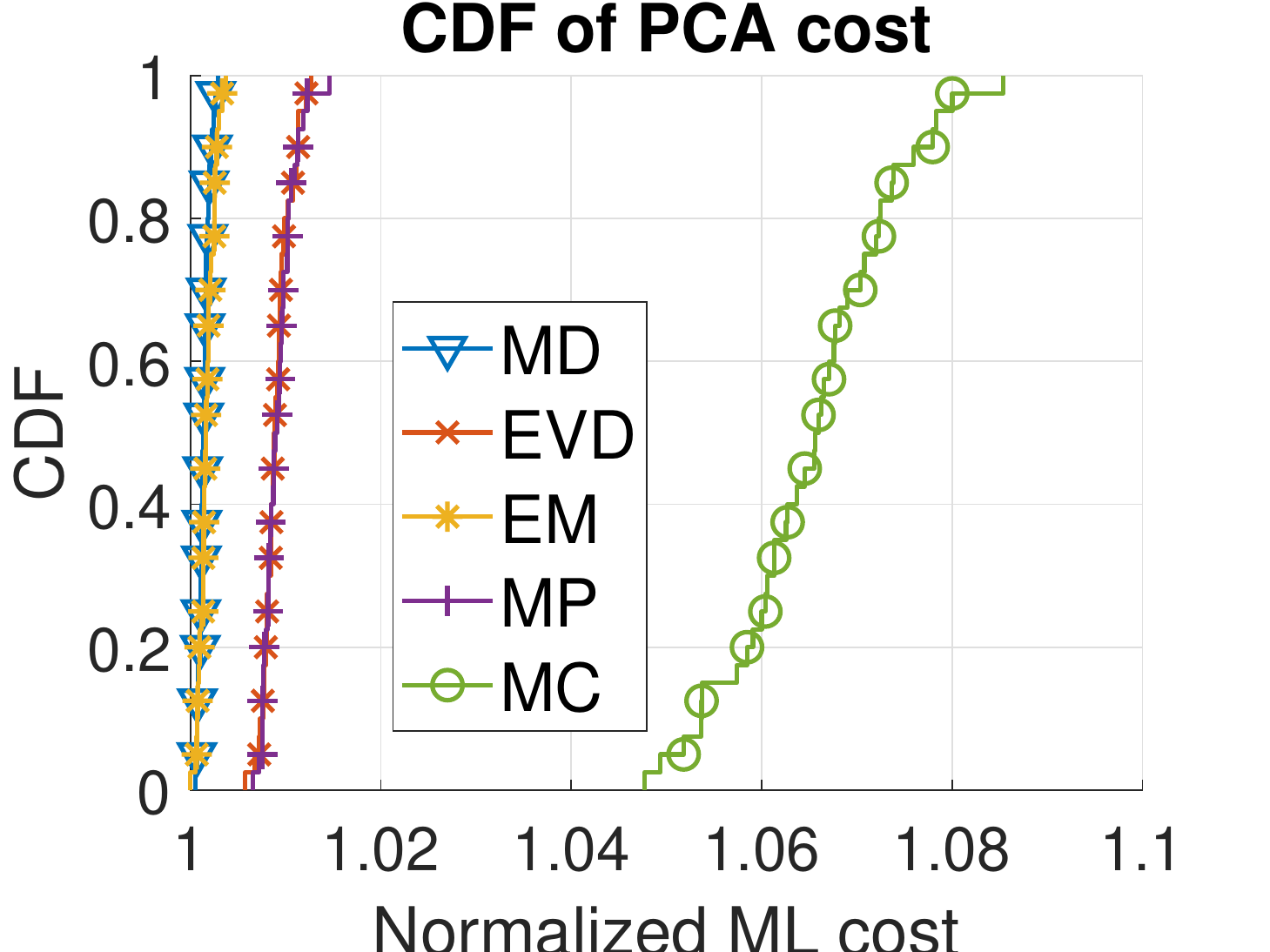}}
    \vspace{-.1em}
    \caption{PCA (11 components)}
\end{subfigure}
 \hspace{-1.5em}
\begin{subfigure}{0.25\textwidth}
    \centerline{
    \includegraphics[width=\textwidth,height=3.2cm]{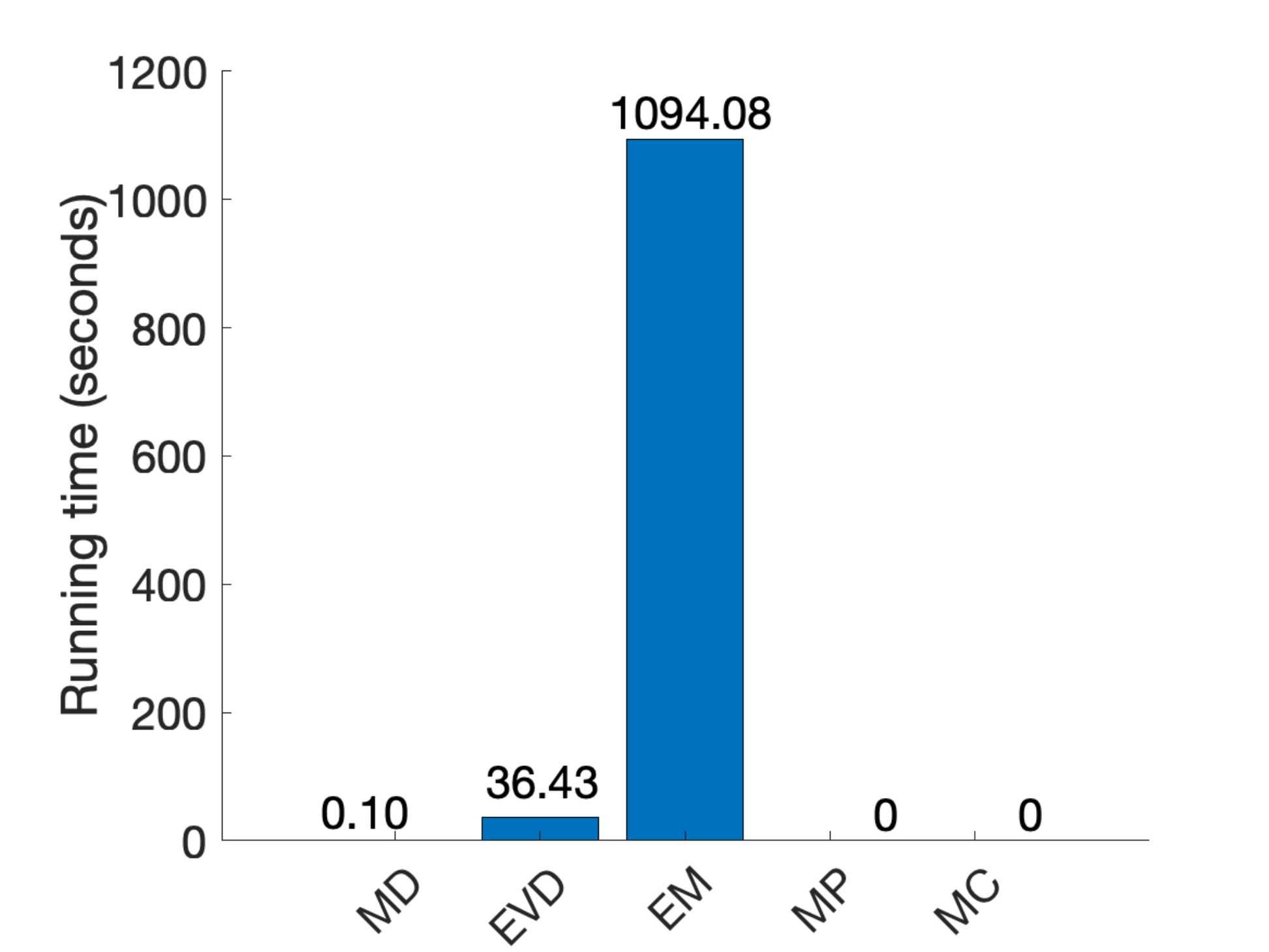}}
    \vspace{-.1em}
    \caption{Running time}
\end{subfigure}
\caption{Evaluation on Pendigits dataset (centralized setting). }
\label{fig:pendigits centralized}
   \vspace{1em}
\end{figure}

\begin{table}[!b]
\footnotesize
\renewcommand{\arraystretch}{1.2}
\caption{Returned $b^*$ for Fisher's Iris } \label{tab: fisher iris}
\centering
\begin{tabular}{ccc}
\hline
Algorithm & $b^*$ & $\#$ of occurrences  \\
\hline
EVD & [31] & [40]\\
\hline
MD & $[16, 18, 20, 23]$ & $[1, 22, 16, 1]$\\
\hline
EM & $[18, 20]$ & $[21, 19]$\\
\hline
\end{tabular}
\end{table}
\normalsize

\begin{table}[!b]
\footnotesize
\renewcommand{\arraystretch}{1.2}
\caption{Returned $b^*$ for Facebook Metric} \label{tab: facebook}
\centering
\begin{tabular}{ccc}
\hline
Algorithm & $b^*$ & $\#$ of occurrences  \\
\hline
EVD & [20] & [40]\\
\hline
MD & $[18,19,20,21,22,23]$ & $[3,5,20,6,4,2]$\\
\hline
EM & $[18,19,20,21,22,23,24,27]$ & $[1,6,9,6,10,5,2,1]$\\
\hline
\end{tabular}
\end{table}
\normalsize

\begin{table}[!b]
\footnotesize
\renewcommand{\arraystretch}{1.2}
\caption{Returned $b^*$ for Pendigits} \label{tab: Pendigits}
\centering
\begin{tabular}{ccc}
\hline
Algorithm & $b^*$ & $\#$ of occurrences  \\
\hline
EVD & [52] & [40]\\
\hline
MD & $[19,20,21]$ & $[4, 35, 1]$\\
\hline
EM & $[19,20,21,22,23,24]$ & $[5,15,6,7,6,1]$\\
\hline
\end{tabular}
\end{table}
\normalsize

\subsubsection{Supervised Learning}
For supervised learning, we evaluate neural network based classification on the MNIST dataset. 
We do not evaluate EM here because its running time for this dataset is prohibitively high. 

Same as unsupervised learning, we only use $2\%$ of the original data, i.e., $B = 60,211,200$. 
Figure~\ref{fig: mnist accuracy} shows the CDFs of classification accuracy over 10 Monte Carlo runs. Note that in contrast to costs, a higher accuracy means better performance. 
MD-MECB, EVD-MECB, and MP all achieve over $80\%$ accuracy, while MC only achieves less than $40\%$ accuracy, because it changes the value of each attribute too much. As we zoom in, we see that MD-MECB performs the best. 
{Moreover, MD-MECB is also relatively fast, with a running time of approximately $15$ minutes per run, whereas EVD-MECB takes up to 
$4.5$ hours for each run due to computing eigenvalue decomposition for a large matrix}.
{After evaluating different budgets,} we note that although MP happens to perform well with $2\%$ data, its performance is highly sensitive to the budget $B$, while the proposed algorithms (MD-MECB, EVD-MECB) adapt well to a wide range of budgets. 


\begin{figure}[!t]
\begin{subfigure}{0.25\textwidth}
\centerline{
\includegraphics[width=\textwidth,height=3.2cm]{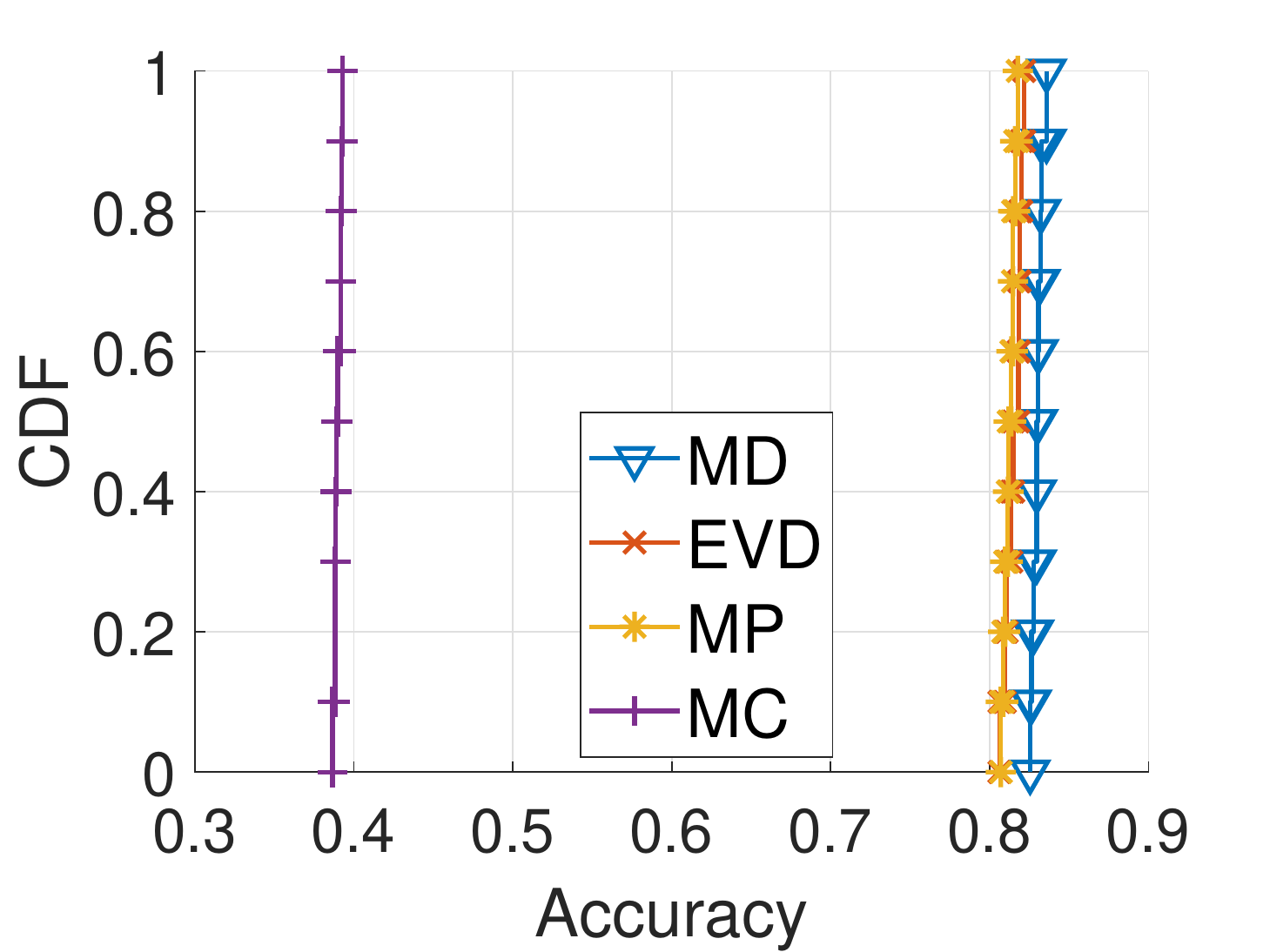}}
\vspace{-.1em}
\caption{Overall CDFs}
\end{subfigure}
 \hspace{-1.5em}
\begin{subfigure}{0.25\textwidth}
\centerline{
\includegraphics[width=\textwidth,height=3.2cm]{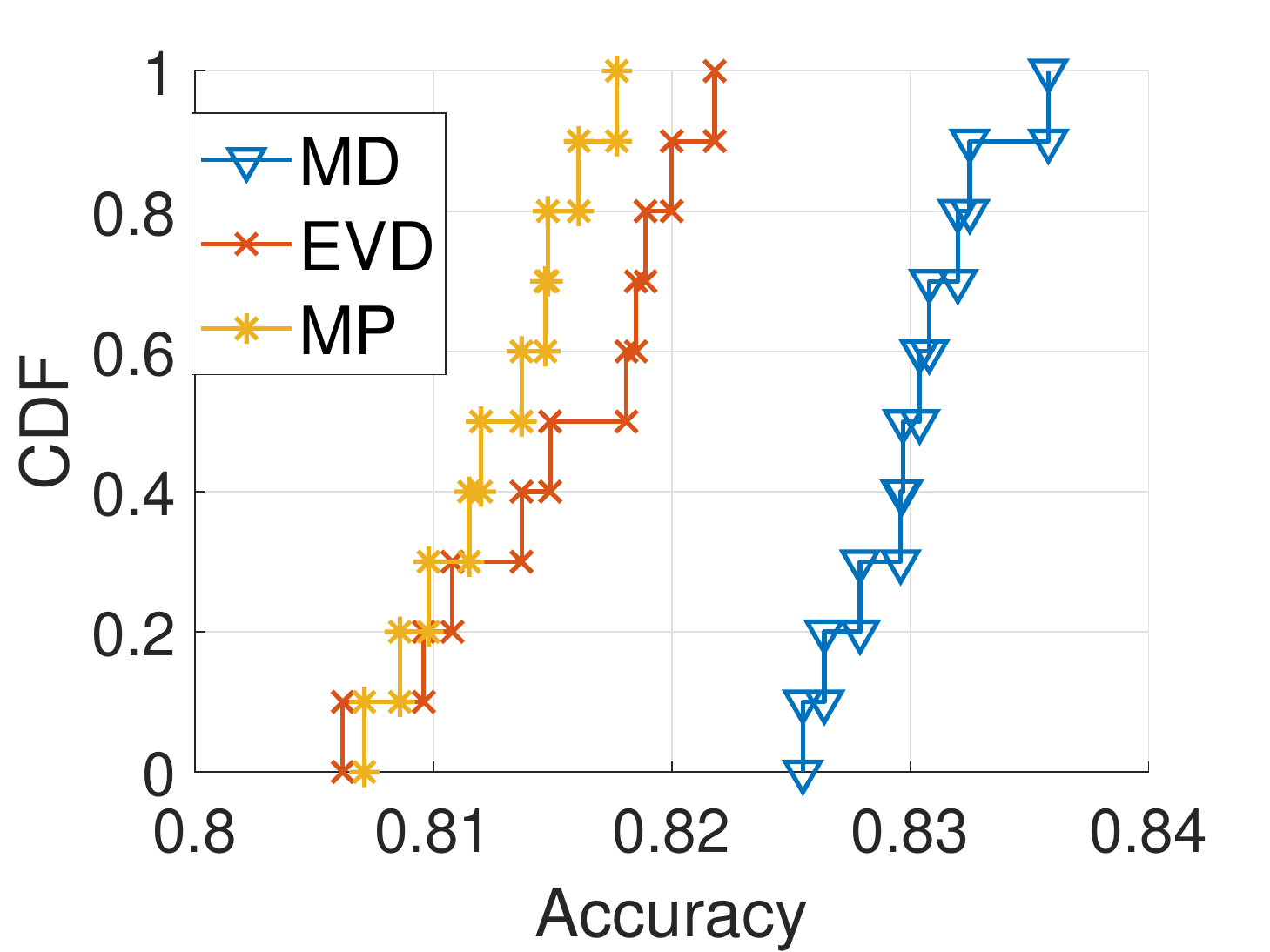}}
\vspace{-.1em}
\caption{Zoomed-in CDFs}
\end{subfigure}
\caption{Evaluation on MNIST (Neural Net Accuracies) }
\label{fig: mnist accuracy}
   \vspace{1em}
\end{figure}


\subsection{Results in Distributed Setting}
In this experiment, we use Fisher's Iris dataset and Pendigits dataset to evaluate our proposed distributed algorithm (Algorithm~\ref{Alg:OBA-RCC}). The original data points are randomly distributed across 10 nodes. 
The global communication budgets are set to $4,875$ bits for Fisher's Iris and $828,087$ bits for Pendigits, which correspond to $10\%$ of the original data size. 

We present the results for distributed setting in Figures~\ref{fig:fisheriris distributed} and \ref{fig:pendigits distributed}, from which we have the following observations: 1) With only $10\%$ data in the distributed setting, most of the algorithms equipped with OBA outperform DRCC with a small degradation in the ML performance. 
2) Compared with OBA-MP, OBA-MC, and DRCC that only rely on one operation to compress the data, our proposed OBA-EVD and OBA-MD which jointly optimize the operations of coreset construction and quantization perform significantly better. 3) OBA-MD  is the most efficient over all these algorithms. 

\subsection{Summary of Experimental Results}
\begin{itemize}
    \item We demonstrate via real ML tasks and datasets that it is possible to achieve reasonable ML performance (less than $10\%$ of degradation {in most cases}) and substantial data reduction ($90$--$98\%$ smaller than the original dataset) by combining coreset construction with quantization.
    \item The proposed algorithms approximate the performance of EM, with a significantly lower running time.\looseness=-1 
    \item Jointly optimizing coreset construction and quantization achieves much better ML performance than relying on only one of these operations.
    \item MD-MECB and its distributed variant (OBA-MD) achieve the best performance-efficiency tradeoff among all the evaluated algorithms, making them the most suitable for large datasets. 
\end{itemize}{}

{}

\begin{figure}[t]
\begin{subfigure}{0.25\textwidth}
\centerline{
\includegraphics[width=\textwidth,height=3.2cm]{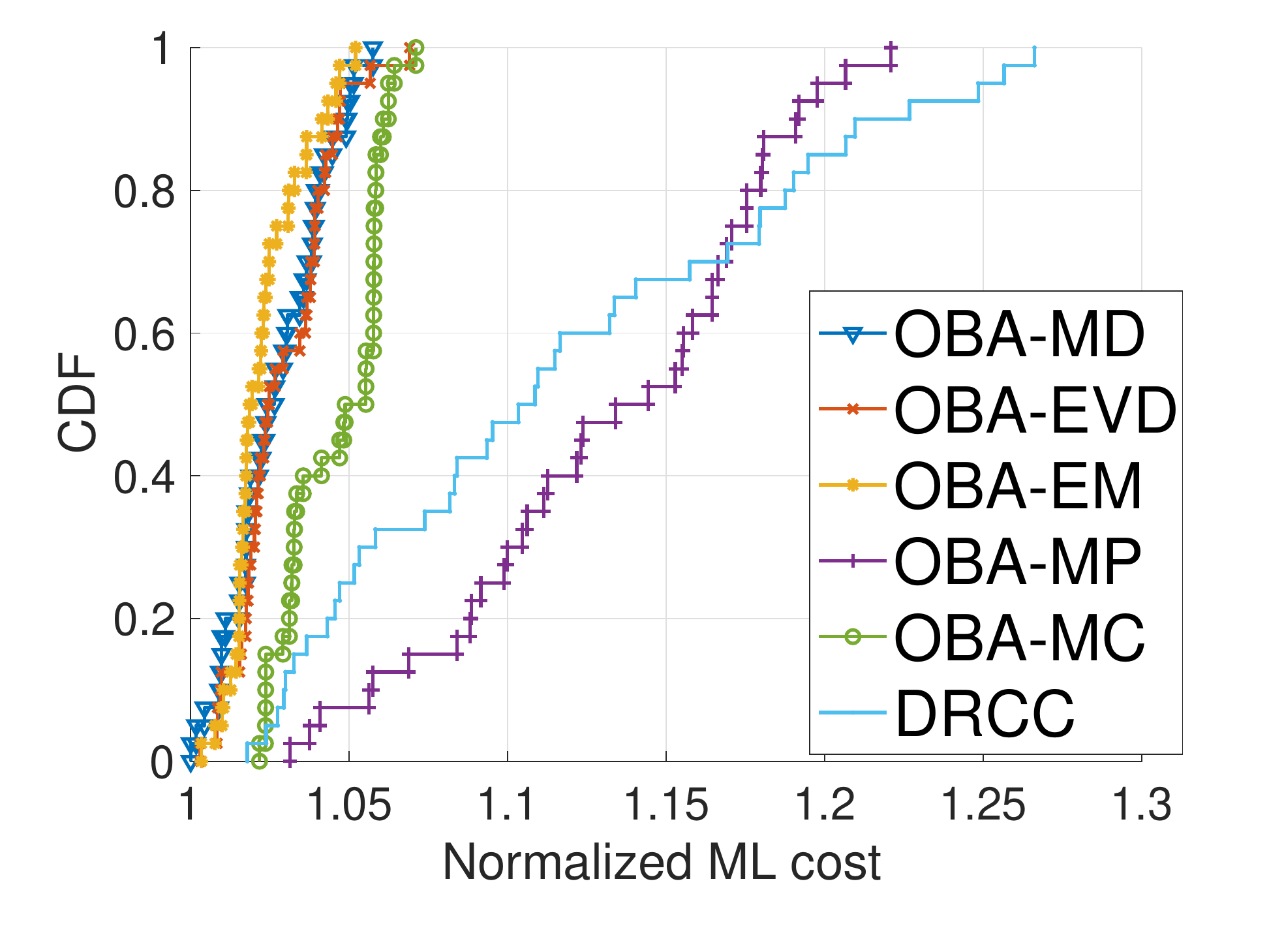}}
\vspace{-.1em}
\caption{MEB}
\end{subfigure}
 \hspace{-6em}
\begin{subfigure}{0.25\textwidth}
\centerline{
\includegraphics[width=\textwidth,height=3.2cm]{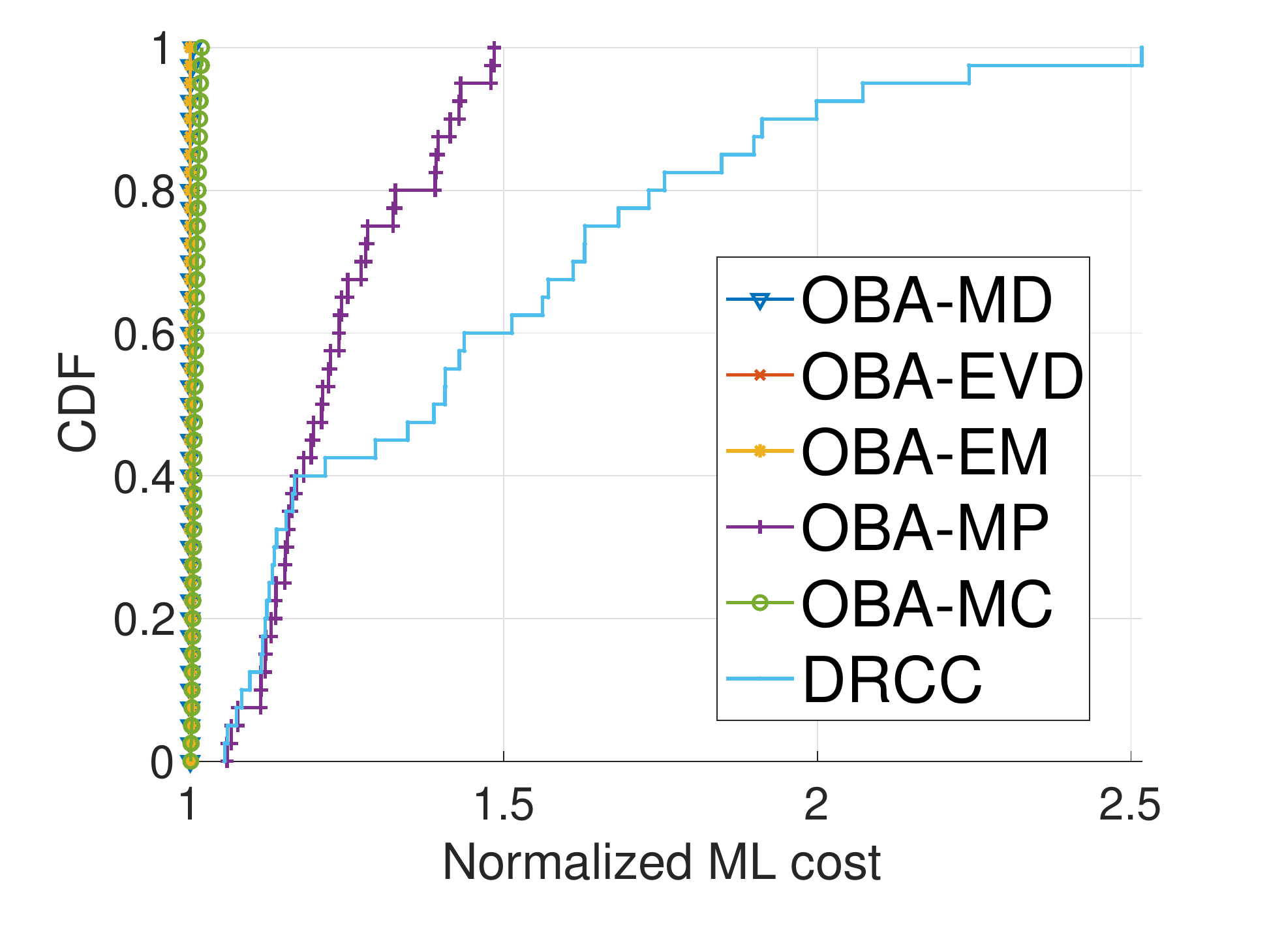}}
\vspace{-.1em}
\caption{$k$-means ($k=2$)}
\end{subfigure}
  \begin{subfigure}{.25\textwidth}
  \centerline{
   \includegraphics[width=\textwidth,,height=3.2cm]{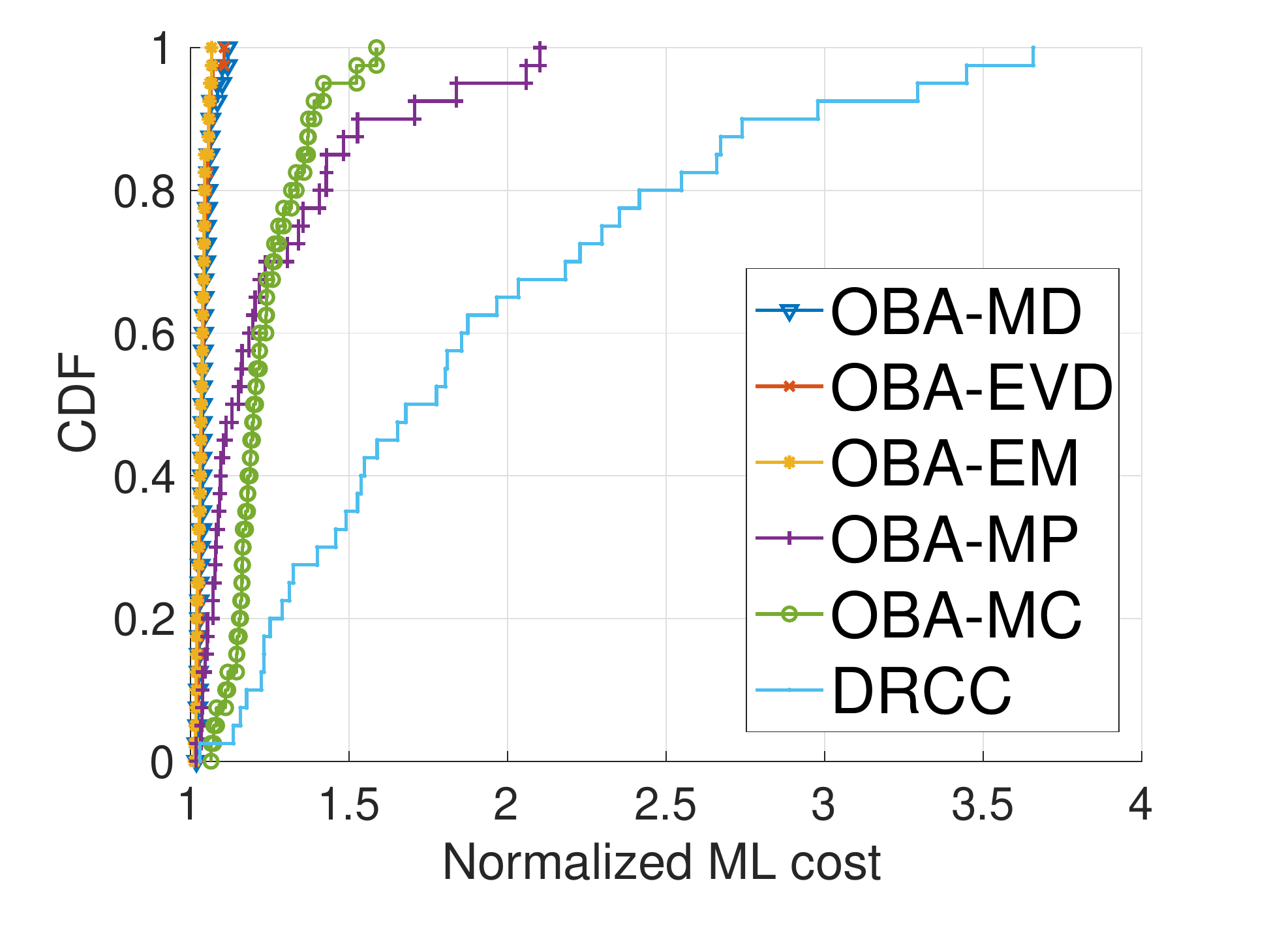}}
   \vspace{-.1em}
    \caption{PCA (3 components) }
  \end{subfigure}
   \hspace{-1.5em}
  \begin{subfigure}{0.25\textwidth}
    \centerline{
  \includegraphics[width=\textwidth,height=3.2cm]{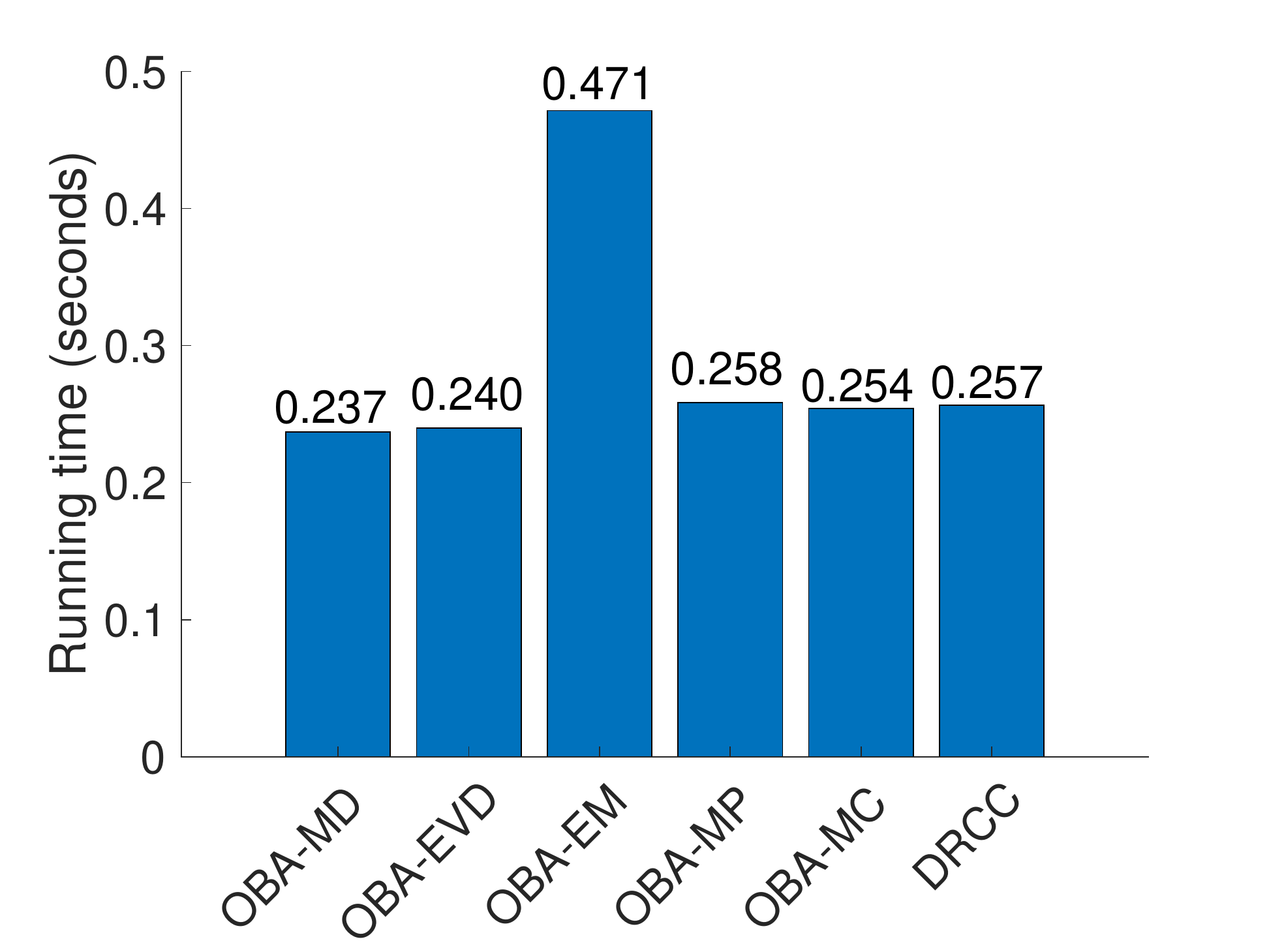}}
  \vspace{-.1em}
\caption{Running time}   
   \end{subfigure}    
\caption{Evaluation on Fisher's Iris dataset (distributed setting). }
\label{fig:fisheriris distributed}
   \vspace{1em}
\end{figure}

\begin{figure}[t]
\begin{subfigure}{0.234\textwidth}
\centerline{
\includegraphics[width=\textwidth,height=3.2cm]{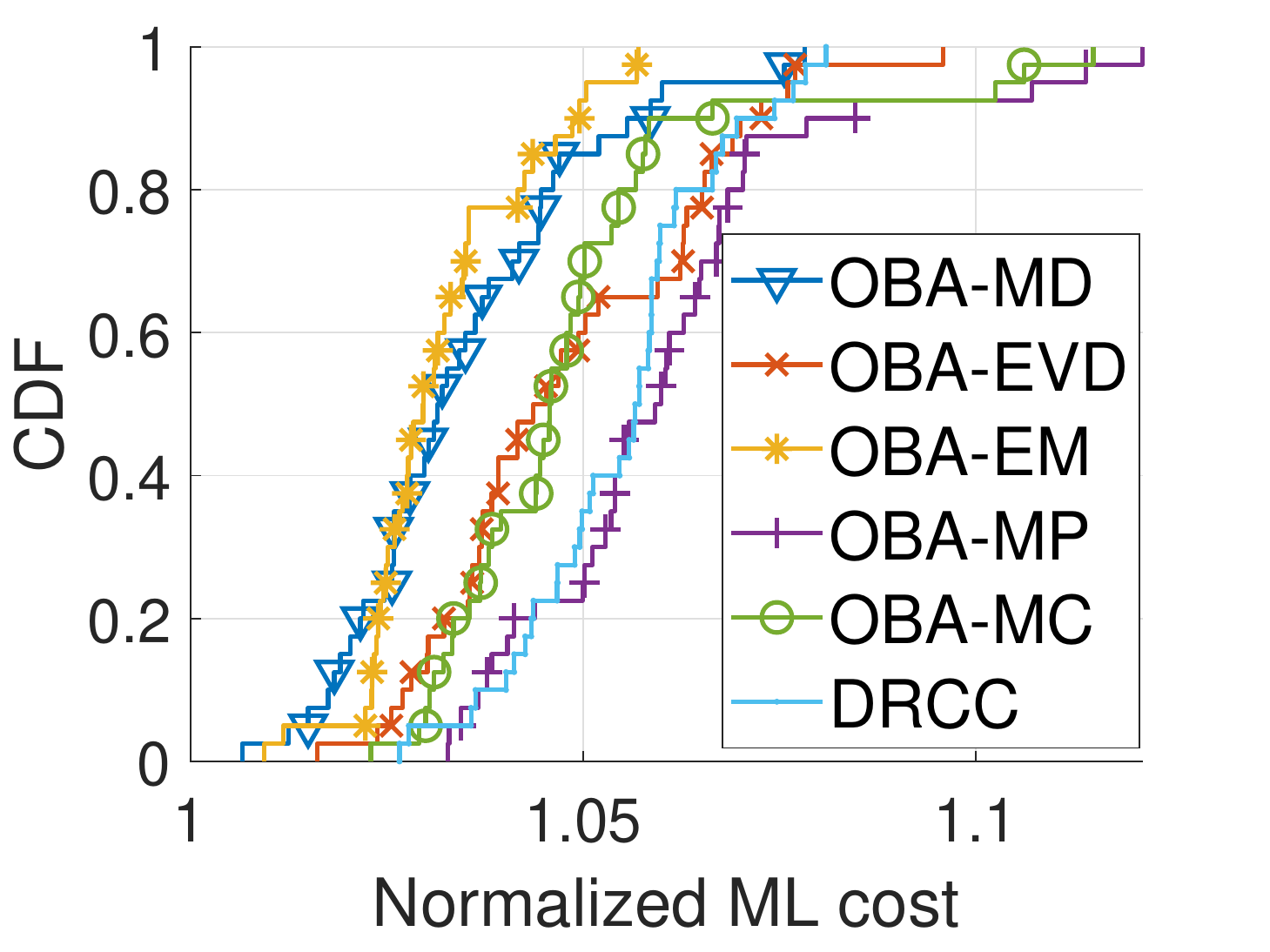}}
\vspace{-.1em}
\caption{MEB}
\end{subfigure}
 \hspace{-6em}
\begin{subfigure}{0.25\textwidth}
\centerline{
\includegraphics[width=\textwidth,height=3.2cm]{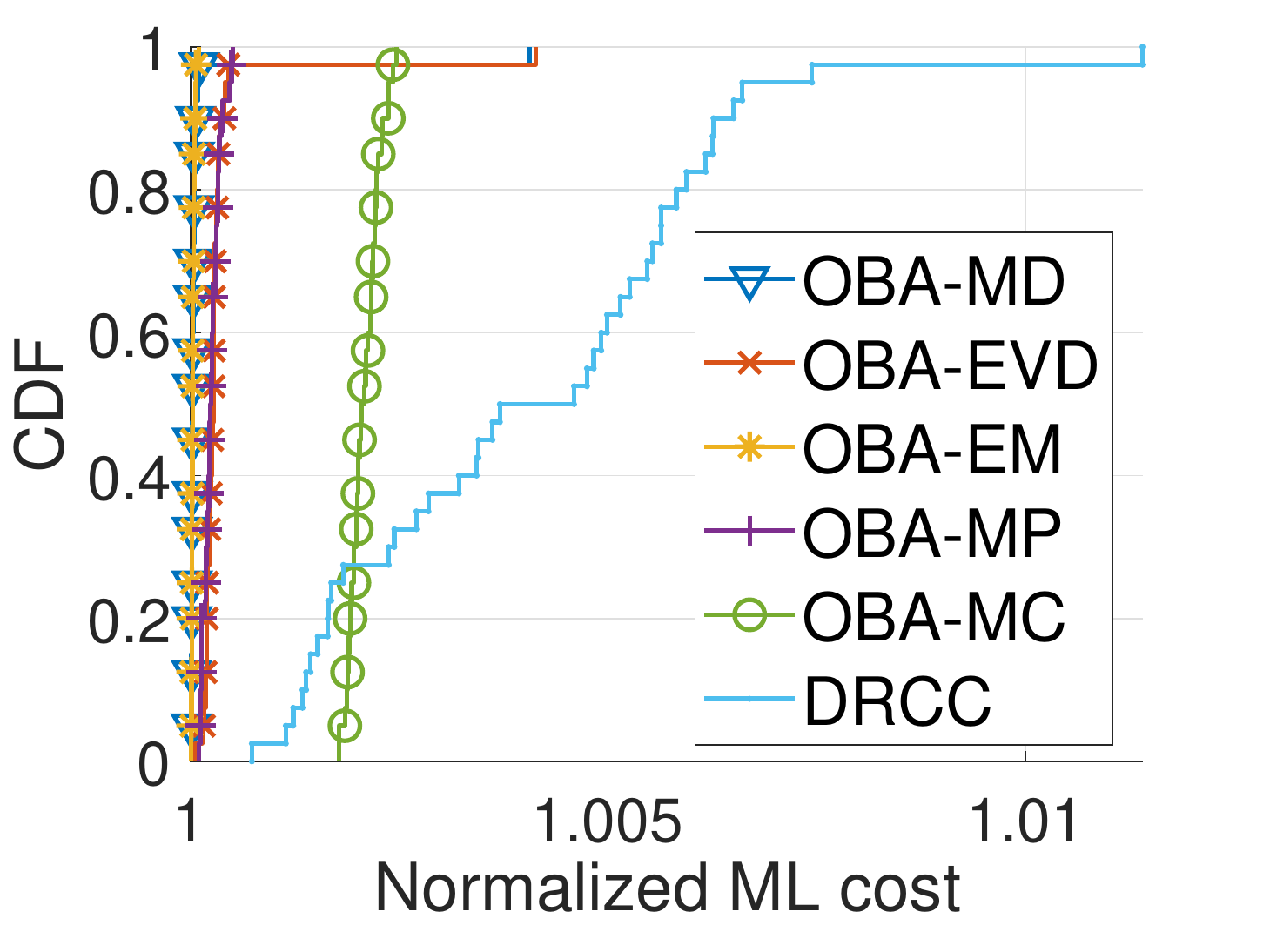}}
\vspace{-.1em}
\caption{$k$-means ($k=2$)}
\end{subfigure}
  \begin{subfigure}{.24\textwidth}
  \centerline{
   \includegraphics[width=\textwidth,,height=3.2cm]{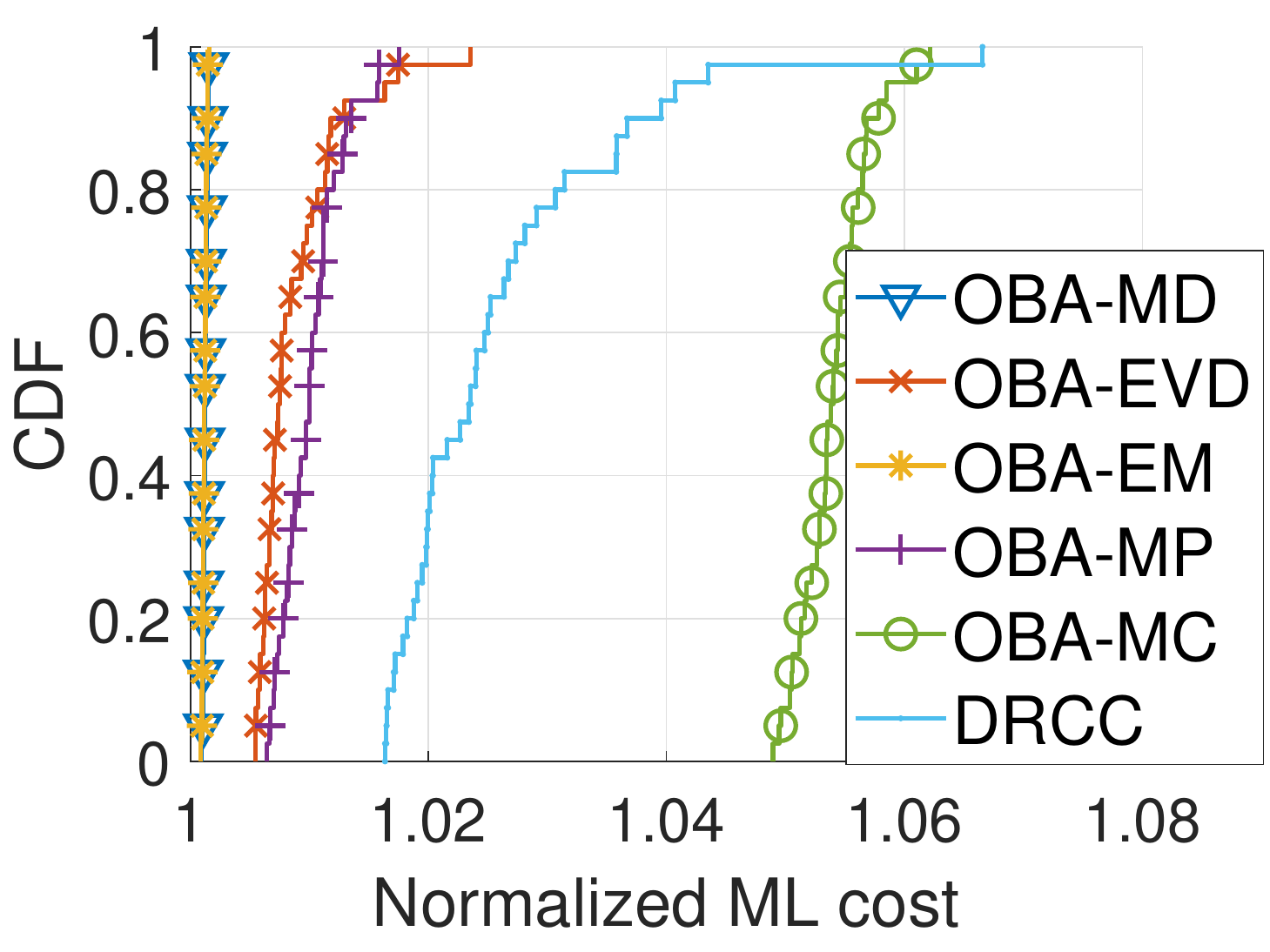}}
   \vspace{-.1em}
    \caption{PCA (11 components) }
  \end{subfigure}
  \begin{subfigure}{0.24\textwidth}
    \centerline{
  \includegraphics[width=\textwidth,height=3.2cm]{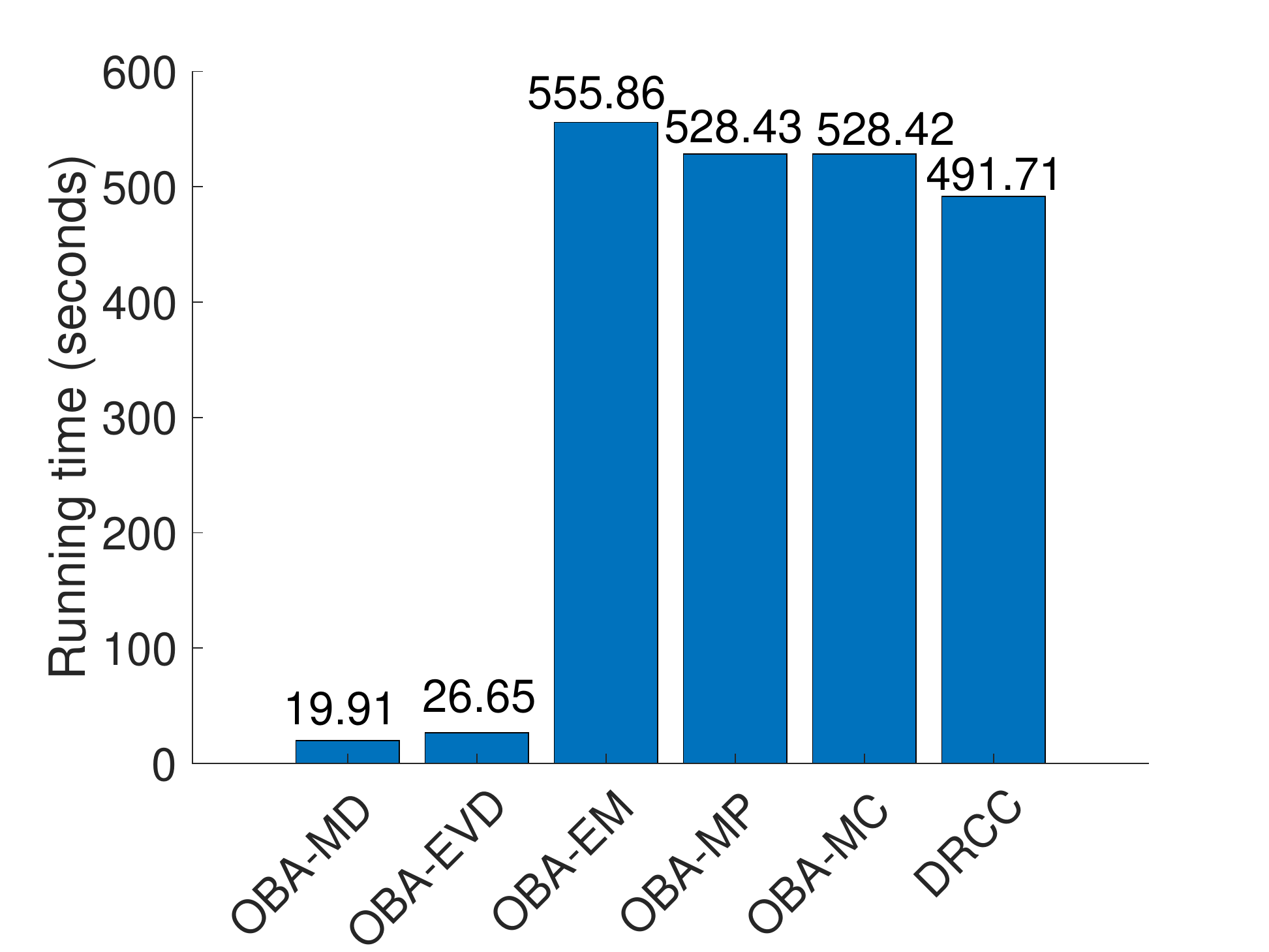}}
  \vspace{-.1em}
\caption{Running time}   
   \end{subfigure}    
\caption{Evaluation on Pendigits dataset (distributed setting). }
\label{fig:pendigits distributed}
   \vspace{1em}
\end{figure}

\section{Conclusion}\label{sec:Conclusion}

In this paper, we have proposed the first framework, MECB, to jointly configure coreset construction algorithms and quantizers in order to minimize the ML error under a given communication budget. We have proposed two algorithms to efficiently compute approximate solutions to the MECB problem, whose effectiveness and efficiency have been demonstrated through experiments based on multiple real-world datasets. We have further proposed an algorithm to extend our solutions to the distributed setting 
by carefully allocating the communication budget across multiple nodes to minimize the overall ML error, which has shown significant improvements  over alternatives when combined with our proposed solutions to MECB. 
Our solutions only depend on a smoothness parameter of the ML cost function, and can thus serve as a key enabler in reducing the communication cost for a broad range of ML tasks. \looseness=-1 



\bibliographystyle{IEEEtran}
\bibliography{coreset_INFOCOM20_sub}

\end{document}